\documentclass{article}

\usepackage{amsmath}
\usepackage{amsfonts}
\usepackage{amssymb}
\usepackage{amsthm}
\usepackage{xcolor}

\usepackage{enumitem}

\usepackage{microtype}
\usepackage{graphicx}
\usepackage{subfigure}
\usepackage{booktabs} 

\usepackage{hyperref}

\usepackage[preprint]{icml2020}

\icmltitlerunning{Revisiting Peng's Q($\lambda$) for Modern Reinforcement Learning}

\usepackage{amsmath,amsfonts,bm}

\usepackage{amsthm}
\usepackage{mleftright}

\newtheorem{theorem}{Theorem}
\newtheorem{corollary}{Corollary}[theorem]
\newtheorem{lemma}[theorem]{Lemma}
\newtheorem{proposition}[theorem]{Proposition}

\newtheorem{assumption}[theorem]{Assumption}
\newtheorem{remark}{Remark}

\newcommand{\E}{\mathbb{E}}

\let\log\relax
\newcommand{\log}{\ln}

\newcommand{\ps}[1]{\Delta_{#1}}
\newcommand{\mapset}[2]{#2^{#1}}  
\newcommand{\const}{\mathbf{1}}  
\newcommand{\deltagreedy}[2]{\set{G}_{#1} (#2)}
\newcommand{\greedy}[1]{\set{G}(#1)}

\newcommand{\set}[1]{\mathbf{#1}}
\let\S\relax
\newcommand{\S}{\set{X}}
\newcommand{\A}{\set{A}}
\newcommand{\SA}{\S \times \A}
\newcommand{\Q}{\R^{\S \times \A}}
\newcommand{\V}{\R^{\S}}

\newcommand{\iplr}[1]{\pp{\gI - \gamma \lambda \gP^{#1}}}
\newcommand{\inviplr}[1]{\iplr{#1}^{-1}}

\newcommand{\abs}[1]{\mleft| #1 \mright|}

\newcommand{\pp}[1]{\mleft( #1 \mright)}
\newcommand{\bb}[1]{\mleft[ #1 \mright]}
\newcommand{\brc}[1]{\mleft\{ #1 \mright\}}
\newcommand{\lp}[2]{\mleft\| #1 \mright\|_{#2}}
\newcommand{\linf}[1]{\lp{#1}{\infty}}
\newcommand{\mb}{\middle \vert}

\def\eqref#1{equation~\ref{#1}}

\def\1{\bm{1}}

\DeclareMathAlphabet{\mathsfit}{\encodingdefault}{\sfdefault}{m}{sl}
\SetMathAlphabet{\mathsfit}{bold}{\encodingdefault}{\sfdefault}{bx}{n}

\def\gA{{\mathcal{A}}}

\def\gI{{\mathcal{I}}}

\def\gM{{\mathcal{M}}}
\def\gN{{\mathcal{N}}}
\def\gO{{\mathcal{O}}}
\def\gP{{\mathcal{P}}}

\def\gR{{\mathcal{R}}}

\def\gT{{\mathcal{T}}}

\newcommand{\R}{\mathbb{R}}

\DeclareMathOperator*{\argmax}{arg\,max}

\begin{document}

\twocolumn[
\icmltitle{Revisiting Peng's Q($\lambda$) for Modern Reinforcement Learning}



\icmlsetsymbol{equal}{*}

\begin{icmlauthorlist}
\icmlauthor{Tadashi Kozuno}{ir,equal}
\icmlauthor{Yunhao Tang}{cu,equal}
\icmlauthor{Mark Rowland}{deepmind}
\icmlauthor{R{\'e}mi Munos}{deepmindparis}
\icmlauthor{Steven Kapturowski}{deepmind}
\icmlauthor{Will Dabney}{deepmind}
\icmlauthor{Michal Valko}{deepmindparis}
\icmlauthor{David Abel}{deepmind}
\end{icmlauthorlist}

\icmlaffiliation{ir}{Independent Researcher, Okayama, Japan}
\icmlaffiliation{cu}{Columbia University, NY, USA}
\icmlaffiliation{deepmind}{DeepMind, London, UK}
\icmlaffiliation{deepmindparis}{DeepMind, Paris, France}

\icmlcorrespondingauthor{Tadashi Kozuno}{tadashi.kozuno@gmail.com}
\icmlcorrespondingauthor{Yunhao Tang}{yt2541@columbia.edu}

\icmlkeywords{Reinforcement Learning, Deep Reinforcement Learning, Peng's Q, ICML}

\vskip 0.3in
]



\printAffiliationsAndNotice{\icmlEqualContribution} 

\begin{abstract}
	Off-policy multi-step reinforcement learning algorithms consist of conservative and non-conservative algorithms: the former actively cut traces, whereas the latter do not. Recently, \citet{munos2016safe} proved the convergence of conservative algorithms to an optimal Q-function. In contrast, non-conservative algorithms are thought to be unsafe and have a limited or no theoretical guarantee. Nonetheless, recent studies have shown that non-conservative algorithms empirically outperform conservative ones. Motivated by the empirical results and the lack of theory, we carry out theoretical analyses of Peng's Q($\lambda$), a representative example of non-conservative algorithms. We prove that \emph{it also converges to an optimal policy} provided that the behavior policy slowly tracks a greedy policy in a way similar to conservative policy iteration. Such a result has been conjectured to be true but has not been proven. We also experiment with Peng's Q($\lambda$) in complex continuous control tasks, confirming that Peng's Q($\lambda$) often outperforms conservative algorithms despite its simplicity. These results indicate that Peng's Q($\lambda$), which was thought to be unsafe, is a theoretically-sound and practically effective algorithm.
\end{abstract}

\section{Introduction}

Q-learning is a canonical algorithm in reinforcement learning (RL) \citep{watkins1989thesis}. It is a single-step algorithm, in that it only uses individual transitions to update value estimates. Many \emph{multi-step} generalisations of Q-learning have been proposed, which allow temporally-extended trajectories to be used in the updating of values \citep{bertsekas1996lambda_pi,watkins1989thesis,peng1994pengq_icml,peng1996pengq_journal,precup2000eligibility,harutyunyan_QLambda_2016,munos2016safe,rowland2020apaptive}, potentially leading to more efficient credit assignment. Indeed, multi-step algorithms have often been observed to outperform single-step algorithms for control in a variety of RL tasks \citep{mousavi2017applying,harb2017investigating,hessel2017rainbow,barth2018distributed,kapturowski2018recurrent,daley2019reconciling}.

However, using multi-step algorithms for RL comes with both theoretical and practical difficulties. The discrepancy between the policy that generated the data to be learnt from (the \emph{behavior policy}) and the policy being learnt about (the \emph{target policy}) can lead to complex, non-convergent behavior in these algorithms, and so must be considered carefully. There are two main approaches to deal with this discrepancy (cf. Table~\ref{table:list of algorithms}). \emph{Conservative methods} ensure convergence is guaranteed no matter what behavior policy is used, typically by truncating the trajectories used for learning. By contrast, \emph{non-conservative methods} typically do not truncate trajectories, and as a result do not come with generic convergence guarantees. Nevertheless, non-conservative methods have consistently been found to outperform conservative methods in practical large-scale applications.
Thus, there is a clear gap in our understanding about non-conservative methods; why do they so work well in practice, but lack the guarantees of their conservative counterparts?

\begin{table*}[t!]
    \centering
    \caption{List of off-policy multi-step algorithms for control. Harutyunyan's Q($\lambda$), Tree-backup, Watkins' Q($\lambda$), and Peng's Q($\lambda$) are abbreviated as HQL, TBL, WQL, and PQL, respectively (cf.\,Section~\ref{section:off-policy multi-step operators} for details of the algorithms). Conservative column indicates if an algorithm is conservative or not (cf.\,Section~\ref{section:conservative and non-conservative algorithms}). Convergence column indicates the convergence of algorithms to any fixed point, whereas Convergence to $Q^*$ column indicates the convergence of algorithms to the optimal Q-function $Q^*$. \textcolor{red}{\checkmark} indicates new results in the present paper. PQL converges to a biased fixed-point when the behavior policy is fixed. It converges to $Q^*$ when a behavior policy is updated appropriately. (An exact condition is given in Section~\ref{section:theoretical analysis}.) \newline}
    \begin{sc}
    \begin{tabular}{l|c|c|c}\toprule[1.5pt]
        \bf Algorithm & \bf Conservative & \bf Convergence & \bf Convergence to $Q^*$ \\\midrule
        $\alpha$-trace \citep{rowland2020apaptive} & No & ? & ?
        \\
        C-trace \citep{rowland2020apaptive} & No & ? & ?
        \\
        HQL \citep{harutyunyan_QLambda_2016} & No & \checkmark (with small $\lambda$) & \checkmark (with small $\lambda$)
        \\
        Retrace \citep{munos2016safe} & Yes & \checkmark & \checkmark
        \\
        TBL \citep{precup2000eligibility} & Yes & \checkmark & \checkmark
        \\ 
        Uncorrected $n$-step Return & No & ? & ?
        \\
        WQL \citep{watkins1989thesis} & Yes & \checkmark & \checkmark
        \\
        PQL \cite{peng1994pengq_icml} & No & \textcolor{red}{\checkmark} (biased) & \textcolor{red}{\checkmark} (cf. caption)
        \\
        \bottomrule[1.46pt]
    \end{tabular}
    \end{sc}
    \label{table:list of algorithms}
\end{table*}

In this paper, we address this question by studying a representative non-conservative algorithm, Peng's Q($\lambda$) \citep[PQL]{peng1994pengq_icml,peng1996pengq_journal}, in more realistic learning settings. Our results show that while PQL does not learn optimal policies under arbitrary behavior policies, a convergence guarantee can be recovered if the behavior policy tracks the target policy, as is often the case in practice. This represents a closing of the gap between the strong empirical performance of non-conservative methods and their previous lack of theoretical guarantees. 

More concretely, our primary theoretical contributions bring new understanding to PQL, and are summarized as follows:
\begin{itemize}[leftmargin=0.4cm,topsep=0pt,itemsep=0pt]
    \item A proof that PQL with a {\em fixed} behavior policy converges to a "biased" (i.e.,~different from $Q^*$) fixed-point.
    \item Analysis of the quality of the resulting policy.
    \item Convergence of PQL to an optimal policy when using appropriate behavior policy updates.
    \item Error propagation analysis when using approximations.
\end{itemize}
In addition to these theoretical insights, we validate the empirical performance of PQL through extensive experiments.
Our focus is on continuous control tasks, where one encounters many technical challenges that do not exist in discrete control tasks (cf. Section~\ref{sec:deeprl}). They are also accessible to a wider range of readers. We show that PQL can be easily extended to popular off-policy actor-critic algorithms such as DDPG, TD3 and SAC \citep{lillicrap2015continuous,fujimoto2018addressing,haarnoja2018soft}. Over a large subset of tasks, PQL consistently outperforms other conservative and non-conservative baseline alternatives.

\section{Notation and Definitions}
\label{section:notation}

For a finite set $\set{A}$ and an arbitrary set $\set{B}$, we let $\ps{\set{A}}$ and $\mapset{\set{A}}{\set{B}}$ be the probability simplex over $\set{A}$ and the set of all mappings from $\set{A}$ to $\set{B}$, respectively.

\paragraph{Markov Decision Processes (MDP).}
We consider an MDP defined by a tuple $\langle \S, \A, \gP, \gP_0, \gR, \gamma \rangle$, where $\S$ is the finite state space, $\A$ the finite action space, $\gP: \SA \rightarrow \ps{\S}$
the state transition probability kernel, $\gP_0 \in \ps{\S}$ the initial state distribution, $\gR$ the (conditional) reward distribution,
and $\gamma \in [0, 1)$ the discount factor \cite{puterman1994mdp}. We let $r \in \Q$ be a reward function defined by $r(x, a) := \int r' \gR(\mathrm{d}r' | x, a)$.

\paragraph{On the Finiteness of the State and Action Spaces.}
While we assume both $\S$ and $\A$ to be finite, most of theoretical results in the paper hold in continuous state spaces with appropriate measure-theoretic considerations.
The finiteness assumption on the action space is necessary to guarantee the existence of the optimal policy \citep{puterman1994mdp}. In Appendix~\ref{appendix:extension to continuous action space}, we discuss assumptions necessary to extend our theoretical results to continuous action spaces.

\paragraph{Policy and Value Functions.}
Suppose a policy $\pi: \S \rightarrow \ps{\A}$. We consider the standard RL setup where an agent interacts with an environment, generating a sequence of state-action-reward tuples $(X_t, A_t, R_t)_{t \geq 0}$ with $A_t$ being an action sampled from some policy; throughout, we denote random variables by upper cases. Define~$G=\sum_{t=0}^\infty \gamma^t R_t$ as the cumulative return.
The state-value and Q-functions are defined by $V^{\pi} (x) := \E \bb{ G \mb X_0=x, \pi }$ and $Q^{\pi} (x, a) := \E \bb{ G \mb X_0=x, A_0=a, \pi }$, respectively, where the conditioning by $\pi$ means $A_t \sim \pi \pp{\cdot \mb X_t}$.

\paragraph{Evaluation and Control.}
Two key tasks in RL are evaluation and control. The problem of evaluation is to learn the Q-function of a fixed policy. The aim in the control setting is to learn an optimal policy $\pi_*$ defined as to satisfy $V^{\pi_*} := V^* \geq V^{\pi},\forall \pi$ (the inequality is point-wise, i.e., $V^* (x) \geq V^{\pi} (x)$ for all $x \in \S$). Similarly to $V^*$, we let $Q^*$ denote the optimal Q-function $Q^{\pi_*}$. As a greedy policy with respect to $Q^*$ is optimal, it suffices to learn $Q^*$. In this paper, we are particularly interested in the off-policy control setting, where an agent collects data with a behavior policy $\mu$, which is not necessarily the agent's current policy $\pi$. On-policy settings are a special case where $\pi=\mu$.

\section{Multi-step RL Algorithms and Operators}
\label{section:operators}

Operators play a crucial role in RL since all value-based RL algorithms (exactly or approximately) update a Q-function based on the recursion $Q_{k+1} := \gO_k Q_k$, where $\gO_k: \Q \rightarrow \Q$ is an operator that characterizes each algorithm. In this section, we review multi-step RL algorithms and their operators.

\paragraph{Basic Operators.} Assume we have a fixed policy $\pi$. With an abuse of notations, we define operators $\pi: \Q \rightarrow \V$ and $\gP: \V \rightarrow \Q$ by
\begin{gather*}
    \pp{\pi Q} (x) := \textstyle\sum_{a \in \A} \pi \pp{ a \mb x } Q \pp{x, a} \text{, and}
    \\
    \pp{\gP V} (x, a) := \textstyle\sum_{y \in \S} \gP \pp{y \mb x, a}  V \pp{y}
\end{gather*}
for any $Q \in \Q$ and $V \in \V$, respectively (hereafter, we omit "for any..." in definitions of operators for brevity). We define their composite $\gP^\pi := \gP \pi$. As a result, the Bellman operator $\gT^{\pi}: \Q \rightarrow \Q$ is defined by $\gT^{\pi} Q := r + \gamma \gP^\pi Q$.
For a function $Q \in \Q$, we let $\greedy{Q}$ be the set of all greedy policies\footnote{Note that there may be multiple greedy policies due to ties.} with respect to $Q$. The Bellman optimality operator $\gT$ is defined by $\gT Q = \gT^{\pi_Q} Q$ with $\pi_Q \in \greedy{Q}$\footnote{Note that this definition is independent of the choice of $\pi_Q$.}.
Q-learning approximates the value iteration (VI) updates $Q_{k+1} := \gT Q_k$. 

\subsection{On-policy Multi-step Operators for Control}

We first introduce on-policy multi-step operators for control.

\paragraph{Modified Policy Iteration (MPI).} MPI uses the recursion $Q_{k+1} := \gT_n^{\pi_k} Q_k$ for Q-function updates \citep{puterman1978mpi}, where $\pi_k \in \greedy{Q_k}$. The $n$-step return operator $\gT_n^\pi: \Q \rightarrow \Q$ is defined by $\gT_n^\pi Q := \pp{\gT^\pi}^n Q$.

\paragraph{$\lambda$-Policy Iteration ($\lambda$-PI).} $\lambda$-PI uses the recursion $Q_{k+1} := \gT_\lambda^{\pi_k} Q_k$ for Q-function updates \cite{bertsekas1996lambda_pi}, where $\pi_k \in \greedy{Q_k}$. The $\lambda$-return operator $\gT_\lambda^\pi: \Q \rightarrow \Q$ is defined as
\begin{align*}
    \gT_\lambda^\pi Q
    &:= (1-\lambda) \textstyle\sum_{n=1}^\infty \lambda^{n-1} \gT_n^\pi Q
    \\
    &= Q + \pp{\gI - \gamma \lambda \gP^\pi }^{-1} \pp{ \gT^\pi Q - Q },
\end{align*}
where $\pp{\gI - \gamma \lambda \gP^\pi }^{-1} := \sum_{t=0}^\infty \pp{ \gamma \lambda \gP^\pi }^t$, and $\lambda \in [0, 1]$.

\subsection{Off-policy Multi-step Operators for Control}
\label{section:off-policy multi-step operators}

Next, we explain off-policy multi-step operators for control. We note that on-policy algorithms in the last subsection can be converted to off-policy versions by using importance sampling \citep{precup2000eligibility,casella2002statistical}.

\paragraph{Uncorrected $n$-step Return.} For a sequence of behavior policies $(\mu)_{k\geq0}$, the uncorrected $n$-step return algorithm uses the recursion $Q_{k+1} := \gN_n^{\mu_k, \pi_k} Q_k$ for Q-function updates \citep{hessel2017rainbow,kapturowski2018recurrent}, where $\pi_k \in \greedy{Q_k}$. Here, the uncorrected $n$-step return operator $\gN_n^{\mu, \pi}$ is defined for any policies $\pi$ and $\mu$ by
\begin{align*}
    \gN_n^{\mu,\pi} Q := \pp{\gT^\mu}^{n-1} \gT^\pi Q.
\end{align*}

\paragraph{Peng's Q($\lambda$) (PQL)}

For a sequence of behavior policies $(\mu)_{k\geq0}$, PQL uses the recursion $Q_{k+1} := \gN_\lambda^{\mu_k, \pi_k} Q_k$ for Q-function updates \citep{peng1994pengq_icml,peng1996pengq_journal}, where $\pi_k \in \greedy{Q_k}$. Here, the PQL operator $\gN_\lambda^{\mu, \pi}$ is defined for any policies $\pi$ and $\mu$ by
\begin{align}\label{eq:PQL operator}
    \gN_\lambda^{\mu,\pi} Q := (1-\lambda) \textstyle\sum_{n=1}^\infty \lambda^{n-1} \gN_n^{\mu, \pi} Q\,,
\end{align}
where $\lambda \in [0, 1]$. Note that \emph{PQL is a generalization of $\lambda$-PI} because it reduces to $\lambda$-PI when $\mu_k = \pi_k$. In other words, PQL is $\lambda$-PI with one additional degree of freedom in $\mu_k$.

\paragraph{General Retrace.}
We next introduce a general version of the Retrace operator \cite{munos2016safe}, from which other operators are obtained as special cases.

For a behavior policy $\mu$ and a target policy $\pi$, we let $\gP^{c \mu}: \Q \rightarrow \Q$ be an operator defined by
\begin{align*}
    \pp{ \gP^{c \mu} Q } (x, a) := \sum_{(y, b) \in \SA} \gP (y | x, a) c (y, b) \mu (b | y) Q (y, b),
\end{align*}
where $c$ is an arbitrary non-negative function over $\SA$ whose choice depends on an algorithm. Note that for any $n$, $( (\gP^{c \mu})^n Q ) (x, a)$ can be estimated off-policy with data collected under the behavior policy $\mu$.

A general Retrace operator $\gR_{\lambda}^{c \mu, \pi}: \Q \rightarrow \Q$ is obtained by replacing $\gP^{\pi}$ of $\pp{\gI - \gamma \lambda \gP^\pi }^{-1}$ in the $\lambda$-return operator $\gT_\lambda^\pi$ with $\gP^{c \mu}$. Concretely,
\begin{align*}
    \gR_{\lambda}^{c\mu, \pi} Q := Q + \pp{\gI - \gamma \lambda \gP^{c\mu} }^{-1} \pp{ \gT^\pi Q - Q}.
\end{align*}
The general Retrace algorithm updates its Q-function by $Q_{k+1} := \gR_{\lambda}^{c_k \mu_k, \pi_k} Q_k$, where $(c_k)_{k \geq 0}$ is a sequence of arbitrary non-negative functions over $\SA$, $(\mu_k)_{k \geq 0}$ is an arbitrary sequence of behavior policies, and $(\pi_k)_{k \geq 0}$ is a sequence of target policies that depends on an algorithm. Given the choices of $c_k$ and $\pi_k$ in Table~\ref{table:off-policy multi-step algorithms}, we recover a few known algorithms \citep{watkins1989thesis,peng1994pengq_icml,peng1996pengq_journal,precup2000eligibility,harutyunyan_QLambda_2016,munos2016safe,rowland2020apaptive}.

The general Retrace algorithm is off-policy as $(\gR_{\lambda}^{c_k \mu_k, \pi_k} Q_k) (x_0, a_0)$ can be estimated off-policy by the following estimator given a trajectory $(x_t, a_t, r_t)_{t\geq0}$ collected under $\mu_k$:
\begin{align}
    Q_k (x_0, a_0) + \sum_{t=0}^\infty \Big( \prod_{u=1}^t c (x_u, a_u)\Big) \gamma^t \lambda^t \delta_t,\label{eq:retrace estimator}
\end{align}
where $\prod_{u=1}^0 c (x_u, a_u) := 1$, and $\delta_t$ is the TD error $r_t + \gamma (\pi_k Q_k) (x_{t+1}) - Q_k (x_t, a_t)$ at time step $t$.

\begin{table}[t!]
    \centering
    \caption{Choices of $c_k$ and $\pi_k$ in off-policy multi-step operators for control. See Section~\ref{section:off-policy multi-step operators} for details. The same abbreviations as those in Table~\ref{table:list of algorithms} are used. For brevity, we defined $\pi_{Q_k} \in \greedy{Q_k}$. We denote $\pi_k (a|x) / \mu_k (a|x)$ by $\rho_k (x, a)$ and $(1-\alpha) + \alpha \pi_{Q_k} (a|x) / \mu_k (a|x)$ by $\widetilde{\rho}_k (x, a)$. $\alpha$-trace and C-trace look the same in the table, but C-trace adaptively changes $\alpha$ so that the trace length matches to a target trace length. \newline}
    \begin{sc}
    \begin{tabular}{c|c|c}\toprule[1.5pt]
        \bf Algorithm & \bf $c_k$ & \bf $\pi_k$ \\\midrule
        $\alpha$-trace & $\min \brc{1, \widetilde{\rho}_k}$ & $\alpha \pi_{Q_k} + (1-\alpha) \mu_k$
        \\
        C-trace & $\min \brc{1, \widetilde{\rho}_k}$ & $\alpha \pi_{Q_k} + (1-\alpha) \mu_k$
        \\
        HQL & $1$ & $\pi_{Q_k}$
        \\ 
        Retrace & $\min \brc{1, \rho_k}$ & Any
        \\
        TBL & $\pi_k$ & Any
        \\
        WQL & $\min \brc{1, \rho_k}$ & $\pi_{Q_k}$
        \\
        PQL & $1$ & $\lambda \pi_{Q_k} + (1-\lambda) \mu_k$
        \\ \bottomrule[1.46pt]
    \end{tabular}
    \end{sc}
    \label{table:off-policy multi-step algorithms}
\end{table}

\section{Conservative and Non-conservative Multi-step RL Algorithms}
\label{section:conservative and non-conservative algorithms}

\citet{munos2016safe} showed that the following conditions suffice for the convergence of the general Retrace to $Q^*$:
\begin{enumerate}[leftmargin=0.5cm,itemsep=0pt,topsep=0pt]
    \item $c_k (x, a) \in [0, \pi_k(a|x) / \mu_k (a|x)]$ for any $k$ and $(x, a) \in \SA$. \item $\pi_k$ satisfies some greediness condition, such as $\varepsilon$-greediness with decreasing $\varepsilon$ as $k$ increases; cf.\,\citet{munos2016safe} for further details.
\end{enumerate}
We call algorithms that satisfy the first condition  \emph{conservative} algorithms for reasons to be explained below. Otherwise, we call the algorithms \emph{non-conservative}. See Table~\ref{table:list of algorithms} for the classification of algorithms. The uncorrected $n$-step return algorithm can also be viewed as a non-conservative algorithm with \emph{non-Markovian traces} that depend also on the past.

\paragraph{Conservativeness, Theoretical Guarantees, and Empirical Performance of Algorithms.} Recall that in the general Retrace update estimator (\ref{eq:retrace estimator}), the effect of the TD error $\delta_t$ is attenuated by $\prod_{u=1}^t c (x_u, a_u)$ in addition to $\gamma^t \lambda^t$. Hence, from the backward view \citep{sutton_RLIntro_1998}, the first condition intuitively requires that \emph{the trace must be cut} if a sub-trajectory $(x_0, a_0, \ldots, x_t, a_t)$ is unlikely under $\pi_k$ relative to $\mu_k$. As a result, conservative algorithms only carry out \emph{safe} updates to Q-functions.

As shown in \citep{munos2016safe}, such conservative updates enable a convergence guarantee of general conservative algorithms. However, \citet{rowland2020apaptive} observed that it often results in frequent trace cuts, and conservative algorithms usually benefit less from multi-step updates.

In contrast, non-conservative algorithms accumulate TD errors without carefully cutting traces. As a result, non-conservative algorithms might perform poorly. As we show later (Proposition~\ref{proposition:HQL oscillation}), it is the case at least for Harutyunyan's Q($\lambda$) (\citet{harutyunyan_QLambda_2016}, HQL), an instance of non-conservative algorithms, when a behavior policy is fixed. 
Nonetheless, non-conservative algorithms are known to perform well in practice \cite{hessel2017rainbow,kapturowski2018recurrent,daley2019reconciling}. To understand its reason, it is important to characterize what kind of updates to the behavior policy entail the convergence of the overall algorithm. In the following sections, we take a step forward along this direction. We establish the convergence guarantee of PQL under two setups: \textbf{(1)} when the behavior policy is fixed; \textbf{(2)} when the behavior policy is updated in an appropriate way.

\section{Theoretical Analysis of Peng's Q($\lambda$)}
\label{section:theoretical analysis}

In this section, we analyze Peng's Q($\lambda$). We start with the \emph{exact case} where there is no update errors in value functions. Later, we will consider the \emph{approximate case} when accounting for update errors. The following lemma is particularly useful in theoretical analyses as well as practical implementations.
\begin{lemma}[\citealp{harutyunyan_QLambda_2016}]\label{lemma:different forms of PQL}
The PQL operator can be rewritten in the following forms:
\begin{align*}
    \gN_{\lambda}^{\mu, \pi} Q
    &= Q + \pp{\gI - \gamma \lambda \gP^{\mu} }^{-1} \pp{ \gT^{\lambda \mu + (1-\lambda) \pi} Q - Q}
    \\
    &= \inviplr{\mu} \pp{ r + \gamma (1-\lambda) \gP^{\pi} Q }\, .
\end{align*}
\end{lemma}

\begin{proof}
This is proven in \citep{harutyunyan_QLambda_2016}, but we provided a proof in Appendix~\ref{appendix:proof of different forms of PQL} for completeness.
\end{proof}

\subsection{Exact Case with a Fixed Behavior Policy}
\label{section:analysis of PQL with a fixed behavior policy}

We now analyze PQL with a fixed behavior policy $\mu$. While the behavior policy is not fixed in a practical situation, the analysis shows a trade-off between bias and convergence rate. This trade-off is analogous to the bias-contraction-rate trade-off of off-policy multi-step algorithms for policy evaluation \citep{rowland2020apaptive} and sheds some light on important properties of PQL.

Concretely, we analyze the following algorithm:
\begin{align}
    \pi_k \in \greedy{Q_k} \text{ and } Q_{k+1} := \gN_{\lambda}^{\mu, \pi_k} Q_k. \label{eq:PQL with fixed behavior}
\end{align}
\citet{harutyunyan_QLambda_2016} has proven that a fixed point of the PQL operator coincides with the unique fixed point of $\lambda \gT^\mu + (1 - \lambda) \gT$, which is guaranteed to exist since $\lambda \gT^\mu + (1 - \lambda) \gT$ is a contraction with modulus $\gamma$ under $L^\infty$-norm (see Appendix~\ref{appendix:preliminaries for theory} for details about the contraction and other notions).

The existence of a fixed point does not imply the convergence of PQL, and we need to show that the distance between $Q_k$ and the fixed point is decreasing. With the following theorem, we show that PQL does converge.
\begin{theorem}\label{theorem:PQL convergence with fixed mu}
Let $\pi_\dagger$ be a policy such that $Q^{\lambda\mu + (1-\lambda)\pi_\dagger} \geq Q^{\lambda\mu + (1-\lambda) \pi}$ for any policy $\pi$, where the inequality is point-wise. 
Then, $\pi_\dagger \in \greedy{Q^{\lambda\mu + (1-\lambda)\pi_\dagger}}$, and $Q_k$ of PQL (\ref{eq:PQL with fixed behavior}) uniformly converges to $Q^{\lambda\mu + (1-\lambda)\pi_\dagger}$ with the rate $\beta^k$, where $\beta := \gamma (1-\lambda) / (1-\gamma\lambda)$.
\end{theorem}

\begin{proof}
See Appendix~\ref{appendix:proof of PQL convergence with fixed mu}.
\end{proof}

We build intuitions about the bias-convergence-rate trade-off implied in Theorem~\ref{theorem:PQL convergence with fixed mu}.
When $\lambda$ increases, the fixed point is $Q^{\lambda \mu + (1-\lambda)\pi_\dagger}$, whose bias against $Q^*$ arguably increases; at the same time, the contraction rate $\beta$ decreases, so that the contraction is faster.

\begin{remark}
In Section~7.6 of \citep{sutton_RLIntro_1998}, it is conjectured that PQL with a fixed policy would converge to a hybrid of $Q^\mu$ and $Q^*$. Theorem~\ref{theorem:PQL convergence with fixed mu} gives an answer to this conjecture and shows that \citet{sutton_RLIntro_1998}'s conjecture is not necessarily true. Rather, the theorem shows that PQL converges to the Q-function of the best policy among policies of the form $\lambda\mu + (1-\lambda)\pi$.
\end{remark}

\subsection{Approximate Case with a Fixed Behavior Policy}
\label{sec:approximate-fixed-mu}

In practice, value-update errors are inevitable due to e.g., finite-sample estimations and function approximation errors. In this subsection, we provide the error propagation analysis of PQL with a fixed behavior policy. As we will see, the analysis depicts a trade-off between fixed point bias and error tolerance.

We analyze the following algorithm:
\begin{align*}
    \pi_k \in \greedy{Q_k}
    \text{ and }
    Q_{k+1} := \gN_{\lambda}^{\mu, \pi_k} Q_k + \varepsilon_k\, ,
\end{align*}
where $\varepsilon_k \in \Q$ denotes the value-update error at iteration $k$. For simplicity, we use $\rho_k := \lambda \mu + (1-\lambda) \pi_k$ and $\rho_\dagger:= \lambda \mu + (1-\lambda) \pi_\dagger$ in this subsection.

In Section~\ref{section:analysis of PQL with a fixed behavior policy}, we showed $\lim_{k\rightarrow\infty} Q_k = Q^{\lambda \mu + (1-\lambda) \pi_\dagger}$ when $\varepsilon_k (x, a) = 0$ at every $(x, a) \in \SA$, and $\pi_\dagger \in \greedy{Q^{\lambda \mu + (1-\lambda) \pi_\dagger}}$. Therefore, $\pi_k$ is an approximation to $\pi_\dagger$, and thus it is natural to define $V^{\rho_\dagger} - V^{\rho_k}$ as the loss of using $\pi_k$ rather than $\pi_\dagger$. The following theorem provides an upper bound for the loss.
\begin{theorem}\label{theorem:error propagation of PQL with a fixed behavior policy}
For any $K$, the following holds:
\begin{align*}
  \linf{V^{\rho_\dagger} - V^{\rho_K}} \leq O(\beta^K) + \frac{2}{1-\gamma} \sum_{k=0}^{K-1} \beta^{K-k-1} \linf{\varepsilon_k}\, ,
\end{align*}
where $\linf{\cdot}$ is the $L_\infty$-norm defined for any real-valued function $f$ by $\linf{f} := \max_v \abs{ f(v) }$.
\end{theorem}

\begin{proof}
See Appendix~\ref{appendix:proof of PQL err prop with fixed behavior}.
\end{proof}

As we have already explained the bias-convergence-rate trade-off, for now we ignore the $O(\beta^K)$ term and focus on the error term. For simplicity, we assume $\linf{\varepsilon_k} = \varepsilon$ for every $k$. Then,
\begin{align*}
  \frac{2}{1-\gamma} \sum_{k=0}^{K-1} \beta^{K-k-1} \linf{\varepsilon_k}
  = O\pp{\frac{1-\gamma\lambda}{(1-\gamma)^2} \varepsilon}\, ,
\end{align*}
In contrast, an analogous result of $\lambda$-PI is $O(\varepsilon / (1-\gamma)^2)$ \citep{scherrer2013lpi}. When $\lambda=0$, these results coincide, which is expected since both $\lambda$-PI and PQL degenerate to value iteration. When $\lambda=1$, PQL's error dependency is $O(\varepsilon / (1-\gamma))$, which is significantly better than $O(\varepsilon / (1-\gamma)^2)$. However in this case, PQL is completely biased and converges to $Q^\mu$. At intermediate values of $\lambda$, PQL achieves a trade-off between error tolerance with bias by changing $\lambda$.

\subsection{Approximate Case with Behavior Policy Updates}

Previously, we have analyzed PQL with a fixed behavior policy. However, in practice, the behavior policy is updated along with the target policy. Besides, value-update errors are inevitable in complex tasks. As a result, PQL may behave quite differently in a practical scenario. This motivates our analysis for the following algorithm:\footnote{This algorithm updates the behavior policy after each application of the PQL operator. In Appendix~\ref{appendix:pi-like pql}, we analyze a case where the behavior policy is updated after multiple applications of the PQL operator.}
\begin{gather}
    Q_{k+1} := \gN_{\lambda}^{\mu_k, \pi_k} Q_k + \varepsilon_k\label{eq:approximate PQL with behavior policy updates}
    \\
    \mu_k := \alpha \pi_k + (1-\alpha) \mu_{k-1},\nonumber
\end{gather}
where $\pi_k \in \greedy{Q_k}$, and $\alpha \in [1-\lambda, 1]$. Note that when $\alpha=1$, this algorithm reduces to $\lambda$-PI as a special case. Though this behavior policy update closely resembles to that of conservative policy iteration \citep{kakade2002cpi}, here we require $\alpha\geq 1-\lambda$.

This algorithm has the following performance guarantee.
\begin{theorem}\label{theorem:error propagation of PQL with behavior policy updates}
For any $K$, the following holds:
\begin{align*}
  \linf{V^* - V^{\pi_K}} \leq O(\zeta^K) + \frac{2}{1-\gamma} \sum_{l=0}^{K-1} \zeta^{K-l-1} \linf{\varepsilon_l}\,,
\end{align*}
where $\zeta := 1-\alpha + \alpha \gamma$. Hence, PQL with behavior policy updates converges to the optimal policy with the rate $\zeta^K$.
\end{theorem}

\begin{proof}
See Appendix~\ref{Appendix:proof of PQL err prop with behavior updates}.
\end{proof}

The first term on the right hand side shows the convergence of PQL with behavior policy updates in an exact case, i.e., $\linf{\varepsilon_k}=0$ for any $k$. It states that the fastest convergence rate is $\gamma^K$ (achieved when $\alpha=1$), which is the same as the convergence rate of VI \citep{munos2005avi}, policy iteration \citep{munos2003api}, MPI \citep{scherrer2012ampi,scherrer2015ampi}, and $\lambda$-PI \citep{scherrer2013lpi}. When $\alpha \neq 1$, the convergence rate coincides with that of conservative policy iteration \citep{scherrer2014cpi}. However we are not aware of a similar result of conservative $\lambda$-PI, which would be an analogue of PQL considered here. Theorem~\ref{theorem:error propagation of PQL with behavior policy updates} also provides the error dependency of PQL (the second term on the right hand side). It coincides with the previous result of the above algorithms when $\alpha=1$, as one would expect, since PQL with $\alpha=1$ is precisely $\lambda$-PI. Nonetheless PQL allows some degree of off-policiness when $\alpha \neq 1$.

\subsection{Oscillatory Behavior of HQL}

In this section, we have proven the convergence of exact PQL (i.e., no value-update errors). However, the following proposition shows that exact HQL, an instance of non-conservative algorithms, does not converge in an MDP when the behavior policy is fixed. Nonetheless, in the same MDP, setting the behavior policy $\mu_k$ to a greedy policy $\pi_k \in \greedy{Q_k}$ guarantees the convergence.
\begin{proposition}\label{proposition:HQL oscillation}
There is an MDP such that when exact HQL is run with a fixed policy $\mu_k = \mu$ for all $k$, $\lambda = 1$, and $Q_0 = Q^\mu$, HQL's Q-function $Q_k$ oscillates between two functions, and its greedy policy $\pi_k$ oscillate between optimal and sub-optimal policies. Contrarily, if $\mu_k \in \greedy{Q_k}$, HQL converges to an optimal policy.
\end{proposition}

\begin{proof}
A proof of the first claim is given in Appendix~\ref{appendix:divergence of HQL}. The second claim immediately follows by noting that if $\mu_k = \pi_k \in \greedy{Q_k}$, HQL is $\lambda$-PI, which is known to converge \citep{bertsekas1996lambda_pi}.
\end{proof}

While this result is specialized to HQL, it sheds light on an important aspect of non-conservative algorithms in general:
\begin{center}
\textbf{While non-conservative algorithms may perform poorly when the behavior policy is fixed, they may converge to $Q^*$ when the behavior policy is updated.}
\end{center}
The above captures a critical aspect of how algorithms behave in practice, where the behavior policy is continuously updated.

\section{Deep RL Implementations}

We next show that Peng's Q($\lambda$) can be conveniently implemented with established off-policy deep RL algorithms. Our experiments focus on continuous control problems where the action space $\A = [-1,1]^m$. A primary motivation for considering continuous control benchmarks (e.g., \citep{brockman2016openai,tassa2018deepmind}) is that they are usually more accessible to a wider RL research community, compared to challenging discrete control benchmarks such as Atari games \citep{bellemare2013arcade}.

\subsection{Off-policy Actor-critic Algorithms}
Off-policy actor-critic algorithms maintain a policy $\pi_\theta(a|x)$ with parameter $\theta$ and a Q-function critic $Q_\phi(x,a)$ with parameter $\phi$. For the policy, a popular choice is the point mass distribution $\pi_\theta(a|x)= \delta(a-\pi_\theta(x))$, where $\pi_\theta(x)\in \mathbb{R}^{\A}$ \citep{lillicrap2015continuous,fujimoto2018addressing,barth2018distributed}. The algorithm collects data with an exploratory behavior policy $\mu$ and saves tuples $(x_t,a_t,r_t)$ into a replay buffer $\mathcal{D}$. At each training iteration, the critic $Q_\phi(x,a)$ is updated by minimizing squared errors against a Q-function target $\mathbb{E}_D\left\lbrack (Q_\phi(x,a)-Q_\text{target}(x,a))^2 \right\rbrack$. The policy is  updated via the deterministic policy gradient $\theta\leftarrow\theta + \alpha \mathbb{E}_\mu \left\lbrack \nabla_\theta  Q_\phi(x,\pi_\theta(x)) \right\rbrack$ \citep{silver2014deterministic}. See further details in Appendix~\ref{appendix:experiment}. 

\subsection{Implementations of Multi-step Operators}

While approximate estimates to $\mathcal{T}Q(x,a)$ are arguably the simplest to implement, it only myopically looks ahead for one step.
Usually, the learning can be significantly sped up when the targets are constructed with multi-step operators. (See, e.g, empirical examples in \citep{hessel2017rainbow,barth2018distributed,kapturowski2018recurrent} and theoretical insights in \citep{rowland2020apaptive}) For example, the uncorrected $n$-step operator is estimated as follows \citep{hessel2017rainbow}: given a $n$-step trajectory $(x_i,a_i,r_i)_{i=0}^n$, the target at $(x_0,a_0)$ is computed as $Q_\text{target}(x_0,a_0)=\sum_{i=0}^{n-1} \gamma^i r_0 + \gamma^n Q_{\phi^-}(x_n,\pi_{\theta^-}(x_n))$. Similar estimates could be derived for all multi-step operators introduced in Section~\ref{section:operators}, especially Peng's Q($\lambda$). We present full details in Appendix~\ref{appendix:experiment}.

\paragraph{Desirable empirical properties of Peng's Q($\lambda$).} The estimates of Peng's Q($\lambda$) do not require importance sampling ratios $\frac{\pi(a|x)}{\mu(a|x)}$. This is especially valuable for continuous control, where the policy could be deterministic, in which case algorithms such as Retrace \citep{munos2016safe} cuts traces immediately. Even when policies are stochastic and traces based on IS ratios are not cut immediately, prior work suggests that the trace cuts are usually pessimistic especially for high-dimensional action space (see, e.g., \citep{wang2016sample} for implementation techniques to mitigate the issue).

\section{Experiments}
\label{section:experiment}

To build better intuitions about  Peng's Q($\lambda$), we start with tabular examples in Section~\ref{sec:toy}. We will see that the empirical properties of Peng's Q($\lambda$) echo the theoretical analysis in previous sections. In Section~\ref{sec:deeprl}, we evaluate
Peng's Q($\lambda$) in the deep RL contexts. We combine Peng's Q($\lambda$) with baseline deep RL algorithms and compare its performance against alternative operators.

\subsection{A tabular example}
\label{sec:toy}
\begin{figure}[h]
    \centering
    \subfigure[Final performance ]{\includegraphics[keepaspectratio,width=.22\textwidth]{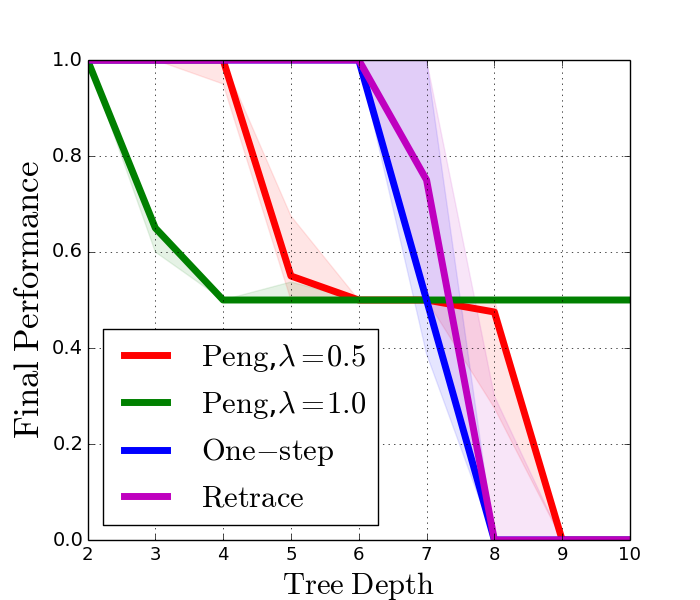}}
    \subfigure[Learning curves]{\includegraphics[keepaspectratio,width=.22\textwidth]{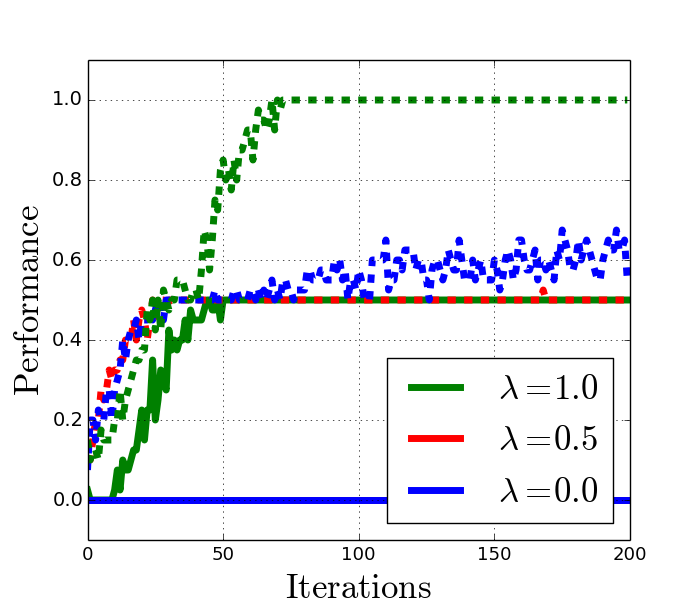}}
    \caption{Performance on tree MDPs. Figure(a) shows how performance changes as a function of three depth $D$; Figure(b) shows the learning curves of different operators. }
    \label{fig:tabular}
\end{figure}

\paragraph{Tree MDP.} We consider toy examples with a tree MDP of depth $D$. The MDPs are binary trees, with each node corresponding to a state. Starting from any non-leaf state, the two actions $a\in\{L,R\}$ transition the agent to one of its child nodes with probability one. Each episode lasts for $D$ steps and the agent always starts at the root node. The rewards are zero everywhere except $r=1$ at the leftmost leaf node and $r=0.5$ at the rightmost leaf node. The behavior policy $\mu$ is $\mu(L|x)=0.3,\mu(R|x)=0.7$ for all states $x$. 

Note that there is a sub-optimal policy of collecting $r=0.5$ at the rightmost leaf. The behavior policy is by design biased towards taking right moves, such that it is easy for the agent to learn the sub-optimal policy. The optimal policy is to take left moves and collect $r=1$. Throughout training, we optimize the target policy $\pi$ while fixing the behavior policy $\mu$. This echos the theoretical setup in Section~\ref{sec:approximate-fixed-mu}. See Appendix~\ref{appendix:experiment} for further details on the setup.

\paragraph{Results.} In Figure~\ref{fig:tabular}(a), we show the converged performance of different algorithms as a function of the MDP's tree depth $D$. When $D=2$, all algorithms achieve the optimal performance; when $\lambda=1$, as $D$ increases, the fixed point bias of Peng's Q($\lambda$) hurts the performance drastically. This is less severe for $\lambda=0.5$, whose performance decays less quickly. On the other hand, both Retrace and the one-step operator learn the optimal policy even for $D\leq 6$. However, when $D$ increases, it becomes difficult to sample the optimal trajectory, making it easy to get trapped with the sub-optimal policy. As such, the sparse rewards make it difficult to learn meaningful Q-functions, unless the return signals get propagated effectively (i.e,. do not cut traces). This is shown in Figure~\ref{fig:tabular}(a), where Peng's Q($\lambda$) with $\lambda=1$ is the only baseline that achieves the sub-optimal performance, while all other algorithms fail to learn anything. 

Similar observations are made in Figure~\ref{fig:tabular}(b), where we compare Peng's Q($\lambda$) for various $\lambda$ under $D=10$ (solid lines) and $D=5$ (dotted lines). Small $\lambda$ corresponds to less bias in the Q-function fixed points, and should asymptotically converge to higher performance; on the other hand, large $\lambda$ suffers sub-optimality when $D$ is small, but gains a substantial advantage when the $D$ is large.

\subsection{Deep RL experiments}
\label{sec:deeprl}

\paragraph{Evaluations.} We evaluate performance over environments with a number of different physics simulation backends, such as MuJoCo \citep{todorov2012mujoco} based DeepMind (DM) control suite \citep{tassa2018deepmind} and an open sourced simulator Bullet physics \citep{coumans2015bullet}. Due to space limit, below we only show results for DM control suite and provide a more complete set of evaluations in Appendix~\ref{appendix:experiment}. 

\paragraph{Baseline comparison.} We use TD3 \citep{fujimoto2018addressing} as the base algorithm. We compare with a few multi-step baselines: \textbf{(1)} one-step (also the base algorithm); \textbf{(2)} Uncorrected $n$-step with a fixed $n$; \textbf{(3)} Peng's Q($\lambda$) with a fixed $\lambda$; \textbf{(4)} Retrace and C-trace. Among all baselines, uncorrected $n$-step operator is the most commonly used non-conservative operator while Retrace is a representative conservative operator. See Appendix~\ref{appendix:experiment} for more details. All algorithms are trained with a fixed number of steps and results are averaged across $5$ random seeds. 

\paragraph{Standard benchmark results.} In the top row of Figure~\ref{fig:offpolicy}, we show evaluations on standard benchmarks. Across most tasks, Peng's Q($\lambda$) performs more stably than other baseline algorithms. We see that Peng's Q($\lambda$) learns generally as fast as other baselines, and in some cases significantly faster than others. 
Note that though Peng's Q($\lambda$) does not necessarily obtain the best learning performance \emph{per each task}, it consistently ranks as the top two algorithms (with ties). This is in contrast to baseline algorithms whose performance rank might vary drastically across tasks. For example, the one-step TD3 performs well in CheetahRun while performs poorly in WalkerWalk. Also, both Ctrace and Retrace generally significantly perform more poorly. We provide further analysis in Appendix~\ref{appendix:experiment}.

\begin{figure*}[h]
    \centering
    \includegraphics[keepaspectratio,width=1.\textwidth]{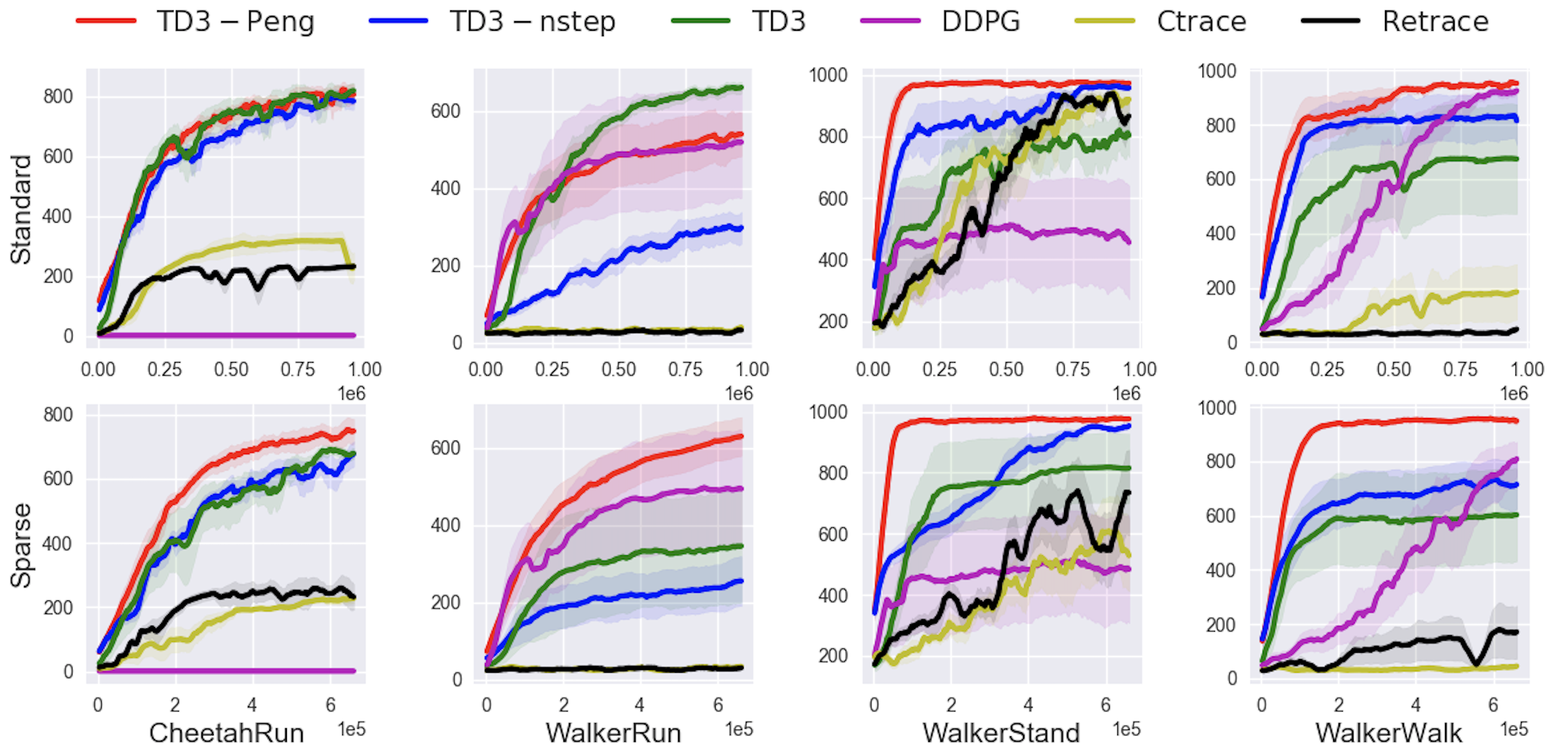}
    \caption{Evaluation of baseline algorithms over standard DM control domains.
    The first row shows results on standard benchmarks; the second row shows results on sparse reward variants of the benchmarks. Four task names are labeled at the bottom. In each plot, x-axis shows the number of training steps and y-axis shows the performance. In standard benchmarks, Peng's Q($\lambda$) generally performs more stably than other algorithms; in sparse reward benchmarks, Peng's Q($\lambda$) outperforms all other algorithms across all presented tasks.}
    \label{fig:offpolicy}
\end{figure*}

\paragraph{Sparse rewards results.} In the bottom row of Figure~\ref{fig:offpolicy}, we show evaluations on sparse reward variants of the benchmark tasks. See details on these environments in Appendix~\ref{appendix:experiment}. Sparse rewards are challenging for deep RL algorithms, as it is more difficult to numerically propagate learning signals across time steps. Accordingly, sparse rewards are natural benchmarks for operator-based algorithms. Across all tasks, Peng's Q($\lambda$) consistently outperforms other baselines. In a few cases, uncorrected $n$-step also outperforms the baseline TD3 -- we speculate that this is because the former propagates the learning signal more efficiently, which is critical for sparse rewards. Compared to uncorrected $n$-step, Peng's Q($\lambda$) seems to achieve a better trade-off between efficient propagation of learning signals  and fixed point biases, which leads to relatively stable and consistent performance gains across all selected benchmark tasks.

\subsection{Additional deep RL experiments}
\paragraph{Maximum-entropy RL.}
In Appendix~\ref{appendix:maxentrl}, we show how Peng's Q($\lambda$) can be extended to maximum-entropy RL \citep{ziebart2008maximum,fox2015taming,haarnoja2017reinforcement,haarnoja2018soft}. We combine multi-step operators with maximum-entropy deep RL algorithms such as SAC \citep{haarnoja2018soft} and show performance gains over benchmark tasks. See Appendix~\ref{appendix:experiment} for further details.

\paragraph{Ablation study on $\lambda$.} In Appendix~\ref{appendix:experiment}, we provide an ablation study on the effect of $\lambda$. We show that the performance of Peng's Q($\lambda$) depends on the choice of $\lambda$. Nevertheless, we find that a single $\lambda$ can usually lead to fairly uniform performance gains across a large number of benchmarks. 

\section{Conclusion}

In this paper, we have studied the non-conservative off-policy algorithm Peng's Q($\lambda$), and shown that while in the worst case its convergence guarantees are less strong than conservative algorithms such as Retrace, convergence guarantees to the optimal policy are recovered when the behavior policy closely tracks the target policy. This has important consequences for deep RL theory and practice, as this condition often holds when agents are trained through replay buffers, and serves to close the gap between the strong empirical performance observed with non-conservative algorithms in deep RL, and their previous lack of theory.

We expect this to have several important consequences for deep RL theory and practice. Firstly, these results make clear that the \emph{degree} of off-policyness is an important quantity that has real impact on the success of deep RL algorithms, and incorporating quantities related to this into the analysis of off-policy algorithms will be important for developing theoretical understanding of deep RL. Secondly, these findings add weight to growing empirical work highlighting that quantities such as replay buffer size and replay ratio are crucial to the success of deep RL agents \citep{zhang2017deeper,daley2019reconciling,fedus2020revisiting}, and deserve further attention.

We believe the analysis presented in this paper is an important step towards a deeper understanding of non-conservative methods, and there are several open questions suitable for future work. For example, the convergence guarantee in Theorem~\ref{theorem:error propagation of PQL with behavior policy updates} requires $\alpha \geq 1-\lambda$. However we conjecture that this assumption can be lifted. Besides, while we did not analyze the concentrability coefficients of PQL, \citet{scherrer2014cpi} reports that conservative policy iteration, which is analogous to PQL, has a better concentrability coefficients. Finally, careful error propagation analyses of gap-increasing algorithms \citep{azar2012dpp,kozuno2019cvi} and policy-update-regularized algorithms \citep{vieillard2020average} show a slow update of policies confer the stability against errors on algorithms. In PQL with behavior policy updates, we expect a similar result when $\alpha$ takes an intermediate value.

\section*{Acknowledgement}

TK was supported by JSPS KAKENHI Grant
Numbers 16H06563. TK thanks Prof.\,Kenji Doya, Dongqi Han, and Ho Ching Chiu at Okinawa Institute of Science and Technology (OIST) for their valuable comments. TK is also grateful to the research support of OIST to the Neural Computation Unit, where TK partially conducted this research. In particular, TK is thankful for OIST's Scientific Computation and Data Analysis section, which maintains a cluster we used for many of our experiments. YHT acknowledges the computational support from Google Cloud Platform.

\bibliographystyle{icml2020}
\bibliography{refs}

\begin{thebibliography}{49}
\providecommand{\natexlab}[1]{#1}
\providecommand{\url}[1]{\texttt{#1}}
\expandafter\ifx\csname urlstyle\endcsname\relax
  \providecommand{\doi}[1]{doi: #1}\else
  \providecommand{\doi}{doi: \begingroup \urlstyle{rm}\Url}\fi

\bibitem[Achiam(2018)]{SpinningUp2018}
Achiam, J.
\newblock {Spinning Up in Deep Reinforcement Learning}.
\newblock 2018.

\bibitem[Asadi \& Littman(2017)Asadi and Littman]{asadi2017alternative}
Asadi, K. and Littman, M.~L.
\newblock {An Alternative Softmax Operator for Reinforcement Learning}.
\newblock In \emph{Proceedings of the International Conference on Machine
  Learning}, 2017.

\bibitem[Azar et~al.(2012)Azar, G{{\'o}}mez, and Kappen]{azar2012dpp}
Azar, M.~G., G{{\'o}}mez, V., and Kappen, H.~J.
\newblock Dynamic policy programming.
\newblock \emph{Journal of Machine Learning Research}, 13\penalty0
  (103):\penalty0 3207--3245, 2012.

\bibitem[Barth-Maron et~al.(2018)Barth-Maron, Hoffman, Budden, Dabney, Horgan,
  TB, Muldal, Heess, and Lillicrap]{barth2018distributed}
Barth-Maron, G., Hoffman, M.~W., Budden, D., Dabney, W., Horgan, D., TB, D.,
  Muldal, A., Heess, N., and Lillicrap, T.
\newblock Distributed distributional deterministic policy gradients.
\newblock In \emph{Proceedings of the International Conference on Learning
  Representations}, 2018.

\bibitem[Bellemare et~al.(2013)Bellemare, Naddaf, Veness, and
  Bowling]{bellemare2013arcade}
Bellemare, M.~G., Naddaf, Y., Veness, J., and Bowling, M.
\newblock {The Arcade Learning Environment: An Evaluation Platform for General
  Agents}.
\newblock \emph{Journal of Artificial Intelligence Research}, 47:\penalty0
  253--279, 2013.

\bibitem[Bertsekas \& Ioffe(1996)Bertsekas and Ioffe]{bertsekas1996lambda_pi}
Bertsekas, D.~P. and Ioffe, S.
\newblock Temporal differences-based policy iteration and applications in
  neuro-dynamic programming.
\newblock Technical Report LIDS-P-2349, Lab. for Info. and Decision Systems
  Report, MIT, Cambridge, Massachusetts, 1996.

\bibitem[Brockman et~al.(2016)Brockman, Cheung, Pettersson, Schneider,
  Schulman, Tang, and Zaremba]{brockman2016openai}
Brockman, G., Cheung, V., Pettersson, L., Schneider, J., Schulman, J., Tang,
  J., and Zaremba, W.
\newblock Open{AI} gym.
\newblock \emph{arXiv preprint arXiv:1606.01540}, 2016.

\bibitem[Casella \& Berger(2002)Casella and Berger]{casella2002statistical}
Casella, G. and Berger, R.~L.
\newblock \emph{{Statistical Inference}}, volume~2.
\newblock Duxbury Pacific Grove, CA, 2002.

\bibitem[Coumans \& Bai(2016--2019)Coumans and Bai]{coumans2015bullet}
Coumans, E. and Bai, Y.
\newblock {PyBullet, a Python module for physics simulation for games, robotics
  and machine learning}.
\newblock \url{http://pybullet.org}, 2016--2019.

\bibitem[Daley \& Amato(2019)Daley and Amato]{daley2019reconciling}
Daley, B. and Amato, C.
\newblock Reconciling $\lambda$-returns with experience replay.
\newblock In \emph{Advances in Neural Information Processing Systems}, 2019.

\bibitem[Fedus et~al.(2020)Fedus, Ramachandran, Agarwal, Bengio, Larochelle,
  Rowland, and Dabney]{fedus2020revisiting}
Fedus, W., Ramachandran, P., Agarwal, R., Bengio, Y., Larochelle, H., Rowland,
  M., and Dabney, W.
\newblock Revisiting fundamentals of experience replay.
\newblock In \emph{Proceedings of the International Conference on Machine
  Learning}, 2020.

\bibitem[Fox et~al.(2016)Fox, Pakman, and Tishby]{fox2015taming}
Fox, R., Pakman, A., and Tishby, N.
\newblock Taming the noise in reinforcement learning via soft updates.
\newblock In \emph{Proceedings of the Conference on Uncertainty in Artificial
  Intelligence}, 2016.

\bibitem[Fujimoto et~al.(2018)Fujimoto, Van~Hoof, and
  Meger]{fujimoto2018addressing}
Fujimoto, S., Van~Hoof, H., and Meger, D.
\newblock {Addressing function approximation error in actor-critic methods}.
\newblock In \emph{Proceedings of the International Conference on Machine
  Learning}, 2018.

\bibitem[Haarnoja et~al.(2017)Haarnoja, Tang, Abbeel, and
  Levine]{haarnoja2017reinforcement}
Haarnoja, T., Tang, H., Abbeel, P., and Levine, S.
\newblock Reinforcement learning with deep energy-based policies.
\newblock In \emph{Proceedings of the International Conference on Machine
  Learning}, 2017.

\bibitem[Haarnoja et~al.(2018)Haarnoja, Zhou, Abbeel, and
  Levine]{haarnoja2018soft}
Haarnoja, T., Zhou, A., Abbeel, P., and Levine, S.
\newblock Soft actor-critic: Off-policy maximum entropy deep reinforcement
  learning with a stochastic actor.
\newblock In \emph{Proceedings of the International Conference on Machine
  Learning}, 2018.

\bibitem[Harb \& Precup(2017)Harb and Precup]{harb2017investigating}
Harb, J. and Precup, D.
\newblock Investigating recurrence and eligibility traces in deep {Q}-networks.
\newblock \emph{arXiv preprint arXiv:1704.05495}, 2017.

\bibitem[Harutyunyan et~al.(2016)Harutyunyan, Bellemare, Stepleton, and
  Munos]{harutyunyan_QLambda_2016}
Harutyunyan, A., Bellemare, M.~G., Stepleton, T., and Munos, R.
\newblock {Q}($\lambda$) with off-policy corrections.
\newblock In \emph{Proceedings of the {International Conference on Algorithmic
  Learning Theory}}, 2016.

\bibitem[Hasselt(2010)]{hasselt2010double}
Hasselt, H.~V.
\newblock Double {Q}-learning.
\newblock In \emph{Advances in Neural Information Processing Systems}, 2010.

\bibitem[Hessel et~al.(2018)Hessel, Modayil, van Hasselt, Schaul, Ostrovski,
  Dabney, Horgan, Piot, Azar, and Silver]{hessel2017rainbow}
Hessel, M., Modayil, J., van Hasselt, H., Schaul, T., Ostrovski, G., Dabney,
  W., Horgan, D., Piot, B., Azar, M.~G., and Silver, D.
\newblock Rainbow: Combining improvements in deep reinforcement learning.
\newblock In \emph{Proceedings of the AAAI Conference on Artificial
  Intelligence}, 2018.

\bibitem[Kakade \& Langford(2002)Kakade and Langford]{kakade2002cpi}
Kakade, S. and Langford, J.
\newblock Approximately optimal approximate reinforcement learning.
\newblock In \emph{Proceedings of the International Conference on Machine
  Learning}, 2002.

\bibitem[Kapturowski et~al.(2018)Kapturowski, Ostrovski, Quan, Munos, and
  Dabney]{kapturowski2018recurrent}
Kapturowski, S., Ostrovski, G., Quan, J., Munos, R., and Dabney, W.
\newblock Recurrent experience replay in distributed reinforcement learning.
\newblock In \emph{Proceedings of the International Conference on Learning
  Representations}, 2018.

\bibitem[Kingma \& Ba(2015)Kingma and Ba]{kingma2014adam}
Kingma, D.~P. and Ba, J.
\newblock Adam: {A} method for stochastic optimization.
\newblock In \emph{Proceedings of the International Conference on Learning
  Representations}, 2015.

\bibitem[Kozuno et~al.(2019)Kozuno, Uchibe, and Doya]{kozuno2019cvi}
Kozuno, T., Uchibe, E., and Doya, K.
\newblock Theoretical analysis of efficiency and robustness of softmax and
  gap-increasing operators in reinforcement learning.
\newblock In \emph{Proceedings of the International Conference on Artificial
  Intelligence and Statistics}, 2019.

\bibitem[Lillicrap et~al.(2016)Lillicrap, Hunt, Pritzel, Heess, Erez, Tassa,
  Silver, and Wierstra]{lillicrap2015continuous}
Lillicrap, T.~P., Hunt, J.~J., Pritzel, A., Heess, N., Erez, T., Tassa, Y.,
  Silver, D., and Wierstra, D.
\newblock Continuous control with deep reinforcement learning.
\newblock In \emph{Proceedings of the International Conference on Learning
  Representations}, 2016.

\bibitem[Mnih et~al.(2015)Mnih, Kavukcuoglu, Silver, Rusu, Veness, Bellemare,
  Graves, Riedmiller, Fidjeland, Ostrovski, et~al.]{mnih2015human}
Mnih, V., Kavukcuoglu, K., Silver, D., Rusu, A.~A., Veness, J., Bellemare,
  M.~G., Graves, A., Riedmiller, M., Fidjeland, A.~K., Ostrovski, G., et~al.
\newblock Human-level control through deep reinforcement learning.
\newblock \emph{Nature}, 518\penalty0 (7540):\penalty0 529--533, 2015.

\bibitem[Mousavi et~al.(2017)Mousavi, Schukat, Howley, and
  Mannion]{mousavi2017applying}
Mousavi, S.~S., Schukat, M., Howley, E., and Mannion, P.
\newblock Applying {Q}($\lambda$)-learning in deep reinforcement learning to
  play {A}tari games.
\newblock In \emph{AAMAS Workshop on Adaptive Learning Agents}, 2017.

\bibitem[Munos(2003)]{munos2003api}
Munos, R.
\newblock Error bounds for approximate policy iteration.
\newblock In \emph{Proceedings of the International Conference on Machine
  Learning}, 2003.

\bibitem[Munos(2005)]{munos2005avi}
Munos, R.
\newblock Error bounds for approximate value iteration.
\newblock In \emph{Proceedings of the AAAI Conference on Artificial
  Intelligence}, 2005.

\bibitem[Munos et~al.(2016)Munos, Stepleton, Harutyunyan, and
  Bellemare]{munos2016safe}
Munos, R., Stepleton, T., Harutyunyan, A., and Bellemare, M.
\newblock Safe and efficient off-policy reinforcement learning.
\newblock In \emph{Advances in Neural Information Processing Systems}, 2016.

\bibitem[Oh et~al.(2018)Oh, Guo, Singh, and Lee]{oh2018self}
Oh, J., Guo, Y., Singh, S., and Lee, H.
\newblock Self-imitation learning.
\newblock In \emph{Proceedings of the International Conference on Machine
  Learning}, 2018.

\bibitem[Peng \& Williams(1994)Peng and Williams]{peng1994pengq_icml}
Peng, J. and Williams, R.~J.
\newblock Incremental multi-step {Q}-learning.
\newblock In \emph{Proceedings of the International Conference on Machine
  Learning}, 1994.

\bibitem[Peng \& Williams(1996)Peng and Williams]{peng1996pengq_journal}
Peng, J. and Williams, R.~J.
\newblock Incremental multi-step {Q}-learning.
\newblock \emph{Machine learning}, 22\penalty0 (1):\penalty0 283--290, March
  1996.

\bibitem[Precup et~al.(2000)Precup, Sutton, and Singh]{precup2000eligibility}
Precup, D., Sutton, R.~S., and Singh, S.~P.
\newblock Eligibility traces for off-policy policy evaluation.
\newblock In \emph{Proceedings of the International Conference on Machine
  Learning}, 2000.

\bibitem[Puterman(1994)]{puterman1994mdp}
Puterman, M.~L.
\newblock \emph{Markov Decision Processes: Discrete Stochastic Dynamic
  Programming}.
\newblock John Wiley \& Sons, Inc., USA, 1st edition, 1994.
\newblock ISBN 0471619779.

\bibitem[Puterman \& Shin(1978)Puterman and Shin]{puterman1978mpi}
Puterman, M.~L. and Shin, M.~C.
\newblock Modified policy iteration algorithms for discounted {M}arkov decision
  problems.
\newblock \emph{Management Science}, 24\penalty0 (11):\penalty0 1127--1137,
  1978.

\bibitem[Rowland et~al.(2020)Rowland, Dabney, and Munos]{rowland2020apaptive}
Rowland, M., Dabney, W., and Munos, R.
\newblock Adaptive trade-offs in off-policy learning.
\newblock In \emph{Proceedings of the International Conference on Artificial
  Intelligence and Statistics}, 2020.

\bibitem[Scherrer(2013)]{scherrer2013lpi}
Scherrer, B.
\newblock Performance bounds for {\(\lambda\)} policy iteration and application
  to the game of {T}etris.
\newblock \emph{Journal of Machine Learning Research}, 14\penalty0
  (1):\penalty0 1181--1227, 2013.

\bibitem[Scherrer(2014)]{scherrer2014cpi}
Scherrer, B.
\newblock Approximate policy iteration schemes: A comparison.
\newblock In \emph{Proceedings of the International Conference on Machine
  Learning}, 2014.

\bibitem[Scherrer et~al.(2012)Scherrer, Gabillon, Ghavamzadeh, and
  Geist]{scherrer2012ampi}
Scherrer, B., Gabillon, V., Ghavamzadeh, M., and Geist, M.
\newblock Approximate modified policy iteration.
\newblock In \emph{Proceedings of the International Conference on Machine
  Learning}, 2012.

\bibitem[Scherrer et~al.(2015)Scherrer, Ghavamzadeh, Gabillon, Lesner, and
  Geist]{scherrer2015ampi}
Scherrer, B., Ghavamzadeh, M., Gabillon, V., Lesner, B., and Geist, M.
\newblock Approximate modified policy iteration and its application to the game
  of {T}etris.
\newblock \emph{Journal of Machine Learning Research}, 16:\penalty0 1629--1676,
  2015.

\bibitem[Silver et~al.(2014)Silver, Lever, Heess, Degris, Wierstra, and
  Riedmiller]{silver2014deterministic}
Silver, D., Lever, G., Heess, N., Degris, T., Wierstra, D., and Riedmiller, M.
\newblock Deterministic policy gradient algorithms.
\newblock In \emph{Proceedings of the International Conference on Machine
  Learning}, 2014.

\bibitem[Sutton \& Barto(1998)Sutton and Barto]{sutton_RLIntro_1998}
Sutton, R.~S. and Barto, A.~G.
\newblock \emph{{Reinforcement Learning: An Introduction}}.
\newblock MIT Press, 1 edition, 1998.

\bibitem[Tassa et~al.(2020)Tassa, Tunyasuvunakool, Muldal, Doron, Liu, Bohez,
  Merel, Erez, Lillicrap, and Heess]{tassa2018deepmind}
Tassa, Y., Tunyasuvunakool, S., Muldal, A., Doron, Y., Liu, S., Bohez, S.,
  Merel, J., Erez, T., Lillicrap, T., and Heess, N.
\newblock {dm{\_}control: Software and Tasks for Continuous Control}, 2020.

\bibitem[Todorov et~al.(2012)Todorov, Erez, and Tassa]{todorov2012mujoco}
Todorov, E., Erez, T., and Tassa, Y.
\newblock Mujoco: A physics engine for model-based control.
\newblock In \emph{Proceedings of the International Conference on Intelligent
  Robots and Systems}, 2012.

\bibitem[Vieillard et~al.(2020)Vieillard, Kozuno, Scherrer, Pietquin, Munos,
  and Geist]{vieillard2020average}
Vieillard, N., Kozuno, T., Scherrer, B., Pietquin, O., Munos, R., and Geist, M.
\newblock Leverage the average: an analysis of {KL} regularization in
  reinforcement learning.
\newblock In \emph{Advances in Neural Information Processing Systems}, 2020.

\bibitem[Wang et~al.(2017)Wang, Bapst, Heess, Mnih, Munos, Kavukcuoglu, and
  de~Freitas]{wang2016sample}
Wang, Z., Bapst, V., Heess, N., Mnih, V., Munos, R., Kavukcuoglu, K., and
  de~Freitas, N.
\newblock Sample efficient actor-critic with experience replay.
\newblock In \emph{Proceedings of the International Conference on Learning
  Representations}, 2017.

\bibitem[Watkins(1989)]{watkins1989thesis}
Watkins, C. J. C.~H.
\newblock \emph{Learning from Delayed Rewards}.
\newblock PhD thesis, University of Cambridge, Cambridge, UK, May 1989.

\bibitem[Zhang \& Sutton(2017)Zhang and Sutton]{zhang2017deeper}
Zhang, S. and Sutton, R.~S.
\newblock A deeper look at experience replay.
\newblock In \emph{NeurIPS Workshop on Deep Reinforcement Learning}, 2017.

\bibitem[Ziebart et~al.(2008)Ziebart, Maas, Bagnell, and
  Dey]{ziebart2008maximum}
Ziebart, B.~D., Maas, A.~L., Bagnell, J.~A., and Dey, A.~K.
\newblock Maximum entropy inverse reinforcement learning.
\newblock In \emph{Proceedings of the AAAI Conference on Artificial
  Intelligence}, 2008.

\end{thebibliography}

\onecolumn
\newpage
\appendix

\section{Preliminaries for Theoretical Analyses}
\label{appendix:preliminaries for theory}

In this appendix, we explain important notions we used in our theoretical analyses. 

\paragraph{Contraction and Monotonicity of Operators.}
An operator $\gO$ from a normed space $(\set{F}, \| \cdot \|)$ to another normed space $(\set{F}', |\cdot|)$ is said to be a contraction if there is a constant $c \in [0, 1)$ such that $\| \gO f - \gO g \| \leq c |f - g |$. This constant $c$ is sometimes called as modulus. For example, $\gT: (\Q, \linf{\cdot}) \rightarrow (\Q, \linf{\cdot})$ is a contraction with modulus $\gamma$. In the main text, we usually meant a contraction under $\linf{\cdot}$ and did not always mention which norm is considered.

A related notion is a non-expansion. If an operator $\gO$ satisfies only $\| \gO f - \gO g \| \leq |f - g|$, it is said to be a non-expansion. For example, $\gP$ is a non-expansion, as proven later.

Monotonicity is probably the most important property in our analyses. An operator $\gO$ is said to be monotone if $\gO f \geq \gO g$ for any $f$ and $g$ satisfying $f \geq g$. For example, $\gP$ is monotone: if $V \geq V'$ (point-wisely, i.e., $V (x) \geq V' (x)$ at every $x$), $\gP V - \gP V'$ holds too, as one can easily confirm from
\begin{align*}
    (\gP V - \gP V')(x, a) = \E \bb{ V(X_1) - V'(X_1) \mb X_0=x, A_0=a} \geq 0.
\end{align*}
Let $\const \in \set{F}$ be a constant function taking $1$ everywhere. If a linear operator $\gO: (\set{F}, \linf{\cdot}) \rightarrow (\set{F}', \linf{\cdot})$ is monotone and satisfies $\gO \const = c \const$ with a scalar $c$, we have $\linf{\gO f - \gO g} \leq c \linf{f - g}$. Indeed,
\begin{align*}
    \gO f - \gO g = \gO (f - g) \leq \gO \linf{ f - g } \const = c \linf{ f - g } \const
    \text{ and } \gO f - \gO g \geq - c \linf{ f - g } \const
\end{align*}
imply $\linf{\gO f - \gO g} \leq c \linf{f - g}$. Thus, $\gP$ is non-expansive as $\mathcal{P}\const = \const$. Note that $(1- \gamma \lambda) \inviplr{\pi}$ is also a non-expansive operator for any $\pi$, as one can easily confirm.

\section{On an Extension of Theoretical Results to Continuous Action Spaces}
\label{appendix:extension to continuous action space}

In this appendix, we explain how to extend our theoretical results to a case where both the state and action spaces are continuous. We mainly follow Appendix~B in \citep{puterman1994mdp}. We ask interested readers to refer to the textbook.

\paragraph{Notation.} Let $\set{S}$ and $\set{S}'$ be Polish spaces. We denote by $\set{B}(\set{S}; c)$ the set of all Borel-measurable functions from $\set{S}$ to a bounded closed interval $[-c, c]$, where $c \in [0, \infty)$; throughout this appendix, the Borel $\sigma$-algebra is always considered. We denote by $\set{P}(\set{S})$ the set of all Borel probability measures on $\set{S}$. We say that a real-valued function $f$ on $\set{S}$ is upper semicontinuous (usc) at a point $p^*$ if $\limsup_{n \rightarrow \infty} f(p_n) \leq f(p^*)$ for any sequence of points $(p_n)_{n\geq0}$ converging to $p^*$. We say that $f$ is usc if it is usc at any point. We denote by $\set{U}(\set{S}; c)$ the set of all usc functions from $\set{S}$ to a bounded closed interval $[-c, c]$, where $c \in [0, \infty)$. We say that a stochastic kernel $q: \set{S} \rightarrow \set{P}(\set{S}')$ is continuous if $\lim_{n \rightarrow \infty} \int f(p') q(dp'|p_n) = \int f(p') q(dp'|p)$ for any bounded continuous function $f$ and any sequence of points $(p_n)_{n\geq0}$ converging to $p$.

\paragraph{Main Discussion.} We impose the following assumption on MDPs. It is necessary to guarantee that all functions in the analyses are usc, as we shall explain soon.
\begin{assumption}\label{assumption:continuity}
The state and action spaces are compact subsets of finite-dimensional Euclidean spaces equipped with Borel $\sigma$-algebras. The reward function $r$ is an usc function bounded by $r_{max}$, and the state transition probability kernel $\gP$ is continuous.
\end{assumption}

We first explain that there exists an optimal policy that is a measurable function from the state space $\S$ to the action space $\A$. Let $V_{max}:=r_{max}/(1-\gamma)$. We denote by $\gM: \set{U}(\SA; V_{max}) \rightarrow \set{U}(\S; V_{max})$ the max operator defined by $(\gM Q)(x) := \max_{a \in \A} Q (x, a)$ for any $Q \in \set{U}(\SA; V_{max})$. Theorem~B.5 in \citet{puterman1994mdp} guarantees that $\gM Q$ is usc. Furthermore, Proposition~B.4 in \citet{puterman1994mdp} guarantees that $\gP \gM Q$ is usc. It is easy to confirm that both $\gM Q$ and $\gP \gM Q$ are bounded by $V_{max}$. Since a sum of usc functions is again usc \citep[Proposition~B.1.a]{puterman1994mdp}, $r + \gamma \gP \gM Q = \gT Q$ belongs to $\set{U}(\SA; V_{max})$. Suppose the recursion $Q_{k+1} := \gT Q_k$. Proposition~B.1.e in \citet{puterman1994mdp} guarantees that $\lim_{k \rightarrow \infty} Q_k = Q^*$ is usc. Proposition~B.4 in \citet{puterman1994mdp} guarantees that there exists a measurable function $\pi_*: \S \rightarrow \A$ such that $Q^* (x, \pi_*(x)) = \max_{a \in \A} Q^* (x, a)$. Accordingly, there exists an optimal policy that is a measurable function from $\S$ to $\A$.

From the above discussion, it is easy to confirm that all $Q_k$ in the exact version of PQL~(\ref{eq:PQL with fixed behavior}) belong to $\set{U}(\SA; V_{max})$ given that the behavior policy $\mu$ is continuous. Therefore, the proof of Theorem~\ref{theorem:PQL convergence with fixed mu} in Appendix~\ref{appendix:proof of PQL convergence with fixed mu} is valid under the assumption that $\mu$ is continuous. We note that it is a weak assumption because the behavior policy $\mu$ is often continuous in practice. Indeed, an action distribution $\mu (\cdot | x)$ is frequently a normal distribution whose mean and diagonal covariance matrix are continuous functions of a state $x$ expressed by, for example, neural networks. As a result, as long as all elements of the diagonal covariance matrix are bounded from below by some constant, the probability density function of $\mu (\cdot | x)$ is bounded. Therefore, the dominated convergence theorem can be used to show that $\mu$ is continuous. When there is an element of the diagonal covariance matrix converging to $0$, this argument does not hold. However, it is a pathological case that usual implementations, such as SpinningUp \citep{SpinningUp2018}, try to avoid by value clipping.

For other theoretical results, we need two additional assumptions: (i) all behavior policies $\mu$ and $\mu_k$ are continuous, and (ii) all error functions $\varepsilon_k$ belong to $\set{U}(\SA; V_{max})$. As for the assumption (i), it is a weak assumption as noted above. (See also the following paragraph on the relaxation of $\pi_k$'s exact greediness.) As for the assumption (ii), it is also a weak assumption: because $Q_{k+1}$ approximates $(\gN_{\lambda}^{\mu})^k Q_0 \in \set{U}(\SA; V_{max})$, there is no strong reason to use a function approximator that does not belong to $\set{U}(\SA; V_{max})$; using a function approximator belonging to $\set{U}(\SA; V_{max})$ guarantees that $\varepsilon_k = Q_{k+1} - \gN_{\lambda}^{\mu} Q_k$ belongs to $\set{U}(\SA; V_{max})$. Similar arguments can be made even when the behavior policy is updated, and we can conclude that these assumptions are weak.

We finally mention how to relax the exact greedy assumption that $\pi_k \in \greedy{Q_k}$. When the action space is continuous, it is not feasible to find an exact greedy policy even if $Q_k$ is continuous. In addition, it is often the case that a policy is expressed by a neural network. However, it is relatively straightforward to extend our theoretical analyses to a case where this exact greedy assumption is relaxed to a $\delta_k$-greedy assumption, that is, $\pi_k Q_k \geq \gM Q_k - \delta_k$, where $\delta_k \in \set{U}(\S; V_{max})$. A similar near-greedy condition is found in, for example, \citet{scherrer2014cpi}.

\section{A Proof of Lemma~\ref{lemma:different forms of PQL} (Different Forms of the PQL Operator)}
\label{appendix:proof of different forms of PQL}

In this appendix, we prove Lemma~\ref{lemma:different forms of PQL}, which provides the following forms of the PQL operator:
\begin{align*}
    \gN_{\lambda}^{\mu, \pi} Q
    &= Q + \pp{\gI - \gamma \lambda \gP^{\mu} }^{-1} \pp{ \gT^{\lambda \mu + (1-\lambda) \pi} Q - Q}
    \\
    &= \inviplr{\mu} \pp{ r + \gamma (1-\lambda) \gP^{\pi} Q }.
\end{align*}

We first recall the original PQL operator (\ref{eq:PQL operator}): $\gN^{\mu, \pi}_\lambda Q := (1-\lambda) \sum_{n=0}^\infty \lambda^n \pp{ \gT^\mu }^n \gT^\pi Q$. Note that each term in the sum can be rewritten as $\pp{ \gT^\mu }^n \gT^\pi Q = \sum_{m=0}^n \gamma^m \pp{ \gP^\mu }^m r + \gamma^{n+1} \pp{ \gP^\mu }^n \gP^\pi Q$. Therefore,
\begin{align*}
    (1-\lambda) \sum_{n=0}^\infty \lambda^n \pp{ \gT^\mu }^n \gT^\pi Q
    &= (1-\lambda) \sum_{n=0}^\infty \lambda^n \bb{ \sum_{m=0}^n \gamma^m \pp{ \gP^\mu }^m r + \gamma^{n+1} \pp{ \gP^\mu }^n \gP^\pi Q }
    \\
    &= (1-\lambda) \sum_{n=0}^\infty \lambda^n \sum_{m=0}^n \gamma^m \pp{ \gP^\mu }^m r + \sum_{n=0}^\infty \lambda^n \gamma^{n+1} (1-\lambda) \pp{ \gP^\mu }^n \gP^\pi Q.
\end{align*}
Note that
\begin{align*}
    (1-\lambda) \sum_{n=0}^\infty \lambda^n \sum_{m=0}^n \gamma^m \pp{ \gP^\mu }^m r
    &=\sum_{n=0}^\infty \lambda^n \sum_{m=0}^n \gamma^m \pp{ \gP^\mu }^m r - \sum_{n=0}^\infty \lambda^{n+1} \sum_{m=0}^n \gamma^m \pp{ \gP^\mu }^m r
    \\
    &= \sum_{n=0}^\infty \lambda^n \sum_{m=0}^n \gamma^m \pp{ \gP^\mu }^m r - \sum_{n=1}^\infty \lambda^n \sum_{m=0}^{n-1} \gamma^m \pp{ \gP^\mu }^m r
    \\
    &= \sum_{n=0}^\infty \lambda^n \gamma^n \pp{ \gP^\mu }^n r.
\end{align*}
Consequently,
\begin{align*}
    (1-\lambda) \sum_{n=0}^\infty \lambda^n \pp{ \gT^\mu }^n \gT Q
    = \sum_{n=0}^\infty \lambda^n \gamma^n \pp{ \gP^\mu }^n \pp{r + \gamma (1-\lambda) \gP^\pi Q}
    = \inviplr{\mu} \pp{r + \gamma (1-\lambda) \gP^\pi Q}.
\end{align*}
The right hand side can be rewritten as follows:
\begin{align*}
    \inviplr{\mu} \pp{r + \gamma (1-\lambda) \gP^\pi Q}
    &= \inviplr{\mu} \pp{ \lambda \gT^\mu Q + (1-\lambda) \gT^{\pi} Q - \lambda \gP^\mu Q}
    \\
    &= \inviplr{\mu} \pp{ \lambda \gT^\mu Q + (1-\lambda) \gT^{\pi} Q - Q + \iplr{\mu} Q}
    \\
    &= Q + \inviplr{\mu} \pp{ \lambda \gT^\mu Q + (1-\lambda) \gT^{\pi} Q - Q}.
\end{align*}
This concludes the proof.

\section{A Proof of Proposition~\ref{proposition:HQL oscillation} (HQL's Oscillation)}
\label{appendix:divergence of HQL}

In this appendix, we prove that under a certain circumstance, HQL oscillates. We prove it by using an example shown in Figure~\ref{fig:HQL divergence}. In this MDP, there are two types of states $\S_1 = \{x | x=1, 2, \ldots \}$ and $\S_2 = \{x' | x=1, 2, \ldots \}$. We denote a state in $\S_1$ by $x$ and a state in $\S_2$ by $x'$. There are two actions $go$ and $exit$. When an agent chooses $go$ at $x$, it moves to $x+1$ with a reward of $-1$. When an agent chooses $exit$ at $x$, it moves to $x'$ with a reward of $1$. At $x'$, any action results in a state transition to the same state $x'$ with a reward of $1$. Therefore, an agent must $exit$ from $x$ as soon as possible.

\begin{figure*}[h]
    \centering
    \includegraphics[keepaspectratio,width=0.6\textwidth]{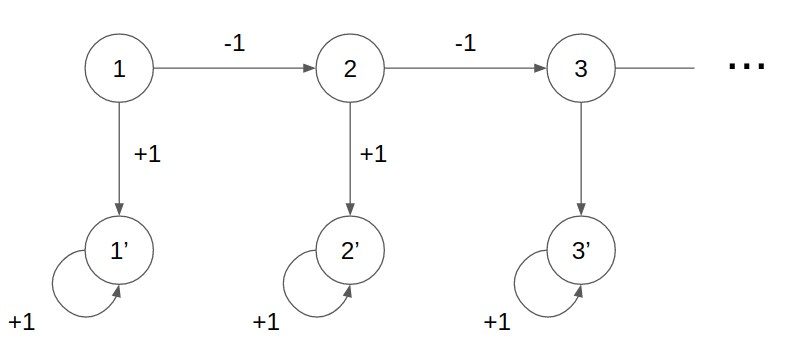}
    \caption{An MDP in which HQL may oscillate.}
    \label{fig:HQL divergence}
\end{figure*}

We assume that $\lambda=1$, $\gamma > 0.5$, $\mu$ chooses $go$ everywhere, $Q_0 (x, exit) = Q^\mu (x, exit) = 1 / \pp{1-\gamma}$, and $Q_0 (x, go) = Q^\mu (x, go) + \delta = - 1 / \pp{1-\gamma} + \delta$ with $\delta > 2 / (1-\gamma)$. For other state-action pairs, $Q_0 = Q^\mu$. As a result, $\pi_0 = \mu$. (At a state $x' \in \S$, any policy is effectively the same as $\mu$.)

\paragraph{Step 1.} HQL's update can be rewritten as follows \citep{harutyunyan_QLambda_2016}:
\begin{align*}
    Q_{k+1}
    := Q_k + \inviplr{\mu} \pp{ \gT^{\pi_k} Q_k - Q_k }
    = Q_k + \inviplr{\mu} \pp{ \gT^\mu Q_k - Q_k + \gamma \gP \pp{\pi_k - \mu} Q_k }.
\end{align*}
Since $\pi_0 = \mu$, and $\mu$ chooses $go$ everywhere (that is, $A_t = go$ for every $t$ in the following equations), we deduce that
\begin{align*}
    Q_1 (x, go)
    = Q_0 (x, go) + \sum_{t=0}^{\infty} \gamma^t \E \bb{ \pp{ \gT^\mu Q_0} (X_t, A_t) - Q_0 (X_t, A_t) \mb X_0=x, A_0=go, \mu}
    = Q^\mu (x, go) = - \frac{1}{1-\gamma}.
\end{align*}
Besides, $Q_1 (x, exit) = Q_0 (x, exit) = Q^\mu (x, exit) = 1/(1-\gamma)$. Accordingly, $\argmax_{a} Q_1 (x, a) = exit$, and $Q_1 = Q^\mu$.

\paragraph{Step 2.} Let us consider what happens at the next iteration. Since $\mu$ chooses $go$ everywhere (that is, $A_t = go$ for every $t$ in the following equations), we deduce that
\begin{align*}
    Q_2 (x, go) 
    &= Q_1 (x, go) + \sum_{t=0}^{\infty} \gamma^t \E \bb{ \pp{ \gT^\mu Q_1} (X_t, A_t) - Q_1 (X_t, A_t) + \gamma \pp{\pi_1 Q_1 - \mu Q_1} (X_{t+1}) \mb X_0=x, A_0=go, \mu}
    \\
    &= Q^\mu (x, go) + \sum_{t=0}^{\infty} \gamma^t \E \bb{ \gamma \pp{\pi_1 Q_1 - \mu Q_1} (X_{t+1}) \mb X_0=x, A_0=go, \mu}
    \\
    &= Q^\mu (x, go) + \frac{2 \gamma}{(1-\gamma)^2} > \frac{1}{1-\gamma} = Q^\mu (x, exit).
\end{align*}
Besides, $Q_2 (x, exit) = Q^\mu (x, exit) = 1/(1-\gamma)$. Accordingly, $\argmax_{a} Q_1 (x, a) = go$.

\paragraph{Step 3.} Now, note that by setting $\delta$ in Step 1 to be $2\gamma / (1-\gamma)^2$, the situation is completely the same as the one we considered in Step 1. Accordingly, $\argmax_{a} Q_3 (x, a) = exit$, and $Q_3 = Q^\mu$. The situation of the next iteration is completely the same as the one we considered in Step 2. This argument can be repeated forever, and thus, $Q_k$ (as well as $\pi_k$) oscillates.

\section{A Proof of Theorem~\ref{theorem:PQL convergence with fixed mu} (PQL's Convergence with a Fixed Behavior Policy)}
\label{appendix:proof of PQL convergence with fixed mu}

We define $\gN_{\lambda}^{\mu}$ as an operator such that $\gN_{\lambda}^{\mu} Q = \gN_{\lambda}^{\mu, \pi_Q} Q$ for any $Q \in \Q$, where $\pi_Q \in \greedy{Q}$. This operator is analogous to $\gT$, whereas $\gN_{\lambda}^{\mu, \pi}$ is analogous to $\gT^\pi$.

From Lemma~\ref{lemma:different forms of PQL}, we deduce that $\gN_{\lambda}^{\mu} Q - \gN_{\lambda}^{\mu} Q' = (1-\lambda) \inviplr{\mu} \pp{ \gT Q - \gT Q' }$ for any $Q, Q' \in \Q$. Because $\inviplr{\mu}$ is linear and monotonic, and satisfies $\inviplr{\mu} \const = \const / (1-\gamma\lambda)$, we have that $\linf{\gN_{\lambda}^{\mu} Q - \gN_{\lambda}^{\mu} Q'} \leq (1-\lambda) \linf{\gT Q - \gT Q'} / (1-\gamma\lambda)$. As noted in Appendix~\ref{appendix:preliminaries for theory}, $\gT$ is a contraction with modulus $\gamma$. Therefore, $\linf{\gN_{\lambda}^{\mu} Q - \gN_{\lambda}^{\mu} Q'} \leq \gamma (1-\lambda) \linf{Q - Q'} / (1-\gamma\lambda) = \beta \linf{Q - Q'}$. Combining this with Banach's fixed point theorem \citep{puterman1994mdp}, it is proven that PQL with a fixed behavior policy converges to a unique fixed point with the rate $\beta^k$.

Let $Q_{\mathrm{fixed}}$ and $\pi_{\mathrm{fixed}}$ be the fixed point and a greedy policy with respect to the fixed point, respectively. (It will turn out to be $Q_{\mathrm{fixed}} = Q^{\lambda \mu + (1-\lambda)\pi_\dagger}$ and $\pi_{\mathrm{fixed}} = \pi_\dagger$.) As noted in Section~\ref{section:analysis of PQL with a fixed behavior policy}, $Q_{\mathrm{fixed}}$ is the fixed point of $\lambda \gT^{\mu} + (1-\lambda) \gT$. It is easy to confirm that it is also the fixed point of $\gT^{\lambda \mu + (1-\lambda) \pi_{\mathrm{fixed}}}$ as $\pi_{\mathrm{fixed}} \in \greedy{Q_{\mathrm{fixed}}}$. Therefore, $Q_{\mathrm{fixed}} = Q^{\lambda \mu + (1-\lambda)\pi_{\mathrm{fixed}}}$.

As $\pi_{\mathrm{fixed}} \in \greedy{Q_{\mathrm{fixed}}}$, $Q_{\mathrm{fixed}} = \gT^{\lambda \mu + (1-\lambda) \pi_{\mathrm{fixed}}} Q_{\mathrm{fixed}} \geq \gT^{\lambda \mu + (1-\lambda) \pi} Q_{\mathrm{fixed}}$ for any policy $\pi$. Therefore, for any positive integer $n$, we have that $Q_{\mathrm{fixed}} \geq (\gT^{\lambda \mu + (1-\lambda) \pi})^n Q_{\mathrm{fixed}}$. As a result, $Q^{\lambda \mu + (1-\lambda)\pi_{\mathrm{fixed}}} = Q_{\mathrm{fixed}} \geq Q^{\lambda \mu + (1-\lambda)\pi}$ for any $\pi$. This implies that $\pi_{\mathrm{fixed}}$ is $\pi_\dagger$.

\section{Double-loop PQL}
\label{appendix:pi-like pql}

In this appendix, we analyze PQL in which $\gN_{\lambda}^{\mu_k}$ is applied multiple times to $Q_k$, and then, the current behavior policy $\mu_k$ is updated to $\mu_{k+1}$. (See Appendix~\ref{appendix:proof of PQL convergence with fixed mu} for the definition of $\gN_{\lambda}^{\mu}$.) Concretely, we consider the following algorithm:
\begin{align}
    \mu_k \in \deltagreedy{\delta_k}{Q_k} \text{ , and } Q_{k+1} := \pp{\gN_{\lambda}^{\mu_k}}^\infty Q_k + \varepsilon_k\,,
\end{align}
where $\delta_k \in \V$ is a non-negative function over $\S$, and $\deltagreedy{\delta_k}{Q_k}$ is the set of $\delta_k$-greedy policies $\pi$ defined by $\pi Q_k \geq \pi' Q_k - \delta_k$ for a greedy policy $\pi' \in \greedy{Q_k}$. Here, we used a shorthand notation $\pp{\gN_{\lambda}^{\mu_k}}^\infty Q_k := \lim_{n\rightarrow\infty} \pp{\gN_{\lambda}^{\mu_k}}^n Q_k$. Note that this algorithm involves a double-loop structure: in the inner loop $\gN_{\lambda}^{\mu_k}$ is repeatedly applied to $(\gN_{\lambda}^{\mu_k})^n Q_k$, and in the outer loop the Q-function and policies are updated. Hence, we call this algorithm as a doule-loop PQL.

There are two main differences from approximate PQL with behavior policy updates (\ref{eq:approximate PQL with behavior policy updates}): first, the behavior policy is required to be near-greedy rather than a mixture policy; second, the Q-function is updated to $\pp{\gN_{\lambda}^{\mu_k}}^\infty Q_k + \varepsilon_k$ rather than $\gN_{\lambda}^{\mu_k} Q_k + \varepsilon_k$. As for the first difference, we think that the behavior policy update in (\ref{eq:approximate PQL with behavior policy updates}) is more practical, but we are unsure if Theorem~\ref{theorem:error propagation of PQL with behavior policy updates} can be extended to double-loop PQL. As for the second difference, this Q-function update is an abstraction of a situation where $\gN_{\lambda}^{\mu_k}$ is applied only finitely many times, and $Q_{k+1}$ deviates from $\pp{\gN_{\lambda}^{\mu_k}}^\infty Q_k$ as a result. Because it is impossible to compute $\pp{\gN_{\lambda}^{\mu_k}}^\infty Q_k$ in a practical situation, this abstraction is necessary. We note that other errors such as function approximation errors can be also included to $\varepsilon_k$.

For this algorithm, we have the following guarantee.
\begin{proposition}
For any non-negative integer $k$, the following holds:
\begin{align*}
    Q^*-Q^{\mu_{k+1}}
    \leq \frac{2\gamma}{1-\gamma} \sum_{j=0}^{k} \gamma^{k-j} \linf{\varepsilon_j}\const + \frac{\gamma (1+\gamma)}{1-\gamma} \sum_{j=0}^{k} \gamma^{k-j} \linf{\delta_{j+1}}\const + \gamma^{k+1} \linf{Q^* - Q^{\mu_0}}\,.
\end{align*}
Thus, if $\linf{\delta_k}\rightarrow 0$ and $\linf{\varepsilon_k}\rightarrow 0$, then $Q^{\mu_k} \rightarrow Q^*$.
\end{proposition}

\begin{proof}
First let us prove that $Q^{\mu_k} - \linf{\varepsilon_k}\const \leq Q_{k+1} \leq Q^{\mu_{k+1}} + \dfrac{1+\gamma}{1-\gamma} \linf{\varepsilon_k}\const + \dfrac{\gamma}{1-\gamma} \linf{\delta_k}\const$. By definition of $Q_{k+1}$,
\begin{align*}
    Q_{k+1} - \varepsilon_k = \pp{\gN_{\lambda}^{\mu_k}}^\infty Q_k \geq Q^{\lambda \mu_k + (1-\lambda) \pi}
    \implies
    Q_{k+1} \geq Q^{\lambda \mu_k + (1-\lambda) \pi} - \linf{\varepsilon_k}\const
\end{align*}
for any policy $\pi$, where the first inequality follows from Theorem~\ref{theorem:PQL convergence with fixed mu}. Now, setting $\pi = \mu_k$ yields $Q_{k+1} \geq Q^{\mu_k} - \linf{\varepsilon_k}\const$. Next, recall that $Q_{k+1}-\varepsilon_k = \pp{\gN_{\lambda}^{\mu_k}}^\infty Q_k$ is a fixed point of $\lambda \gT^{\mu_k} + (1-\lambda) \gT$. Accordingly,
\begin{align*}
    Q_{k+1}-\varepsilon_k
    = \lambda \gT^{\mu_k} (Q_{k+1}-\varepsilon_k) + (1-\lambda) \gT (Q_{k+1}-\varepsilon_k)
    \leq \gT (Q_{k+1}-\varepsilon_k)
    \leq \gT Q_{k+1} + \gamma \linf{\varepsilon_k}\const\,,
\end{align*}
where the last inequality follows from the monotonicity of $\gT$ and $-\varepsilon_k \leq \linf{\varepsilon_k}$. Furthermore, from the fact that $\mu_{k+1} \in \deltagreedy{\delta_{k+1}}{Q_{k+1}}$, we deduce that $Q_{k+1}-\varepsilon_k \leq \gT^{\mu_{k+1}} Q_{k+1} + \gamma \linf{\varepsilon_k}\const + \gamma \linf{\delta_k}\const$. This implies that $Q_{k+1} \leq \gT^{\mu_{k+1}} Q_{k+1} + (1+\gamma) \linf{\varepsilon_k}\const + \gamma \linf{\delta_k}\const$. By induction on $k$ and the monotinicity of $\gT^{\mu_{k+1}}$, we deduce that
\begin{align*}
    Q_{k+1} \leq Q^{\mu_{k+1}} + \dfrac{1+\gamma}{1-\gamma} \linf{\varepsilon_k}\const + \frac{\gamma}{1-\gamma} \linf{\delta_k}\const\,.
\end{align*}

Now we have
\begin{align*}
    Q^*-Q^{\mu_{k+1}}
    &=\gamma\gP^{\pi_*} Q^* - \gamma\gP^{\pi_*} Q_{k+1} + \underbrace{\gamma\gP^{\pi_*} Q_{k+1} - \gamma\gP^{\mu_{k+1}} Q_{k+1}}_{\leq \gamma \linf{\delta_{k+1}}\const} + \gamma\gP^{\mu_{k+1}} Q_{k+1} - \gamma\gP^{\mu_{k+1}} Q^{\mu_{k+1}}
    \\
    &\leq \gamma\gP^{\pi_*} \pp{Q^* - Q_{k+1}} + \gamma \linf{\delta_{k+1}}\const + \gamma\gP^{\mu_{k+1}} \underbrace{\pp{ Q_{k+1} - Q^{\mu_{k+1}}} }_{\leq \frac{1+\gamma}{1-\gamma} \linf{\varepsilon_k}\const + \frac{\gamma}{1-\gamma} \linf{\delta_{k+1}}\const}
    \\
    &\leq \frac{\gamma (1+\gamma)}{1-\gamma} \linf{\varepsilon_k}\const + \frac{\gamma (1+\gamma)}{1-\gamma} \linf{\delta_{k+1}}\const + \gamma\gP^{\pi_*} \pp{Q^* - Q_{k+1}}
    \\
    &\leq \frac{2\gamma}{1-\gamma} \linf{\varepsilon_k}\const + \frac{\gamma (1+\gamma)}{1-\gamma} \linf{\delta_{k+1}}\const + \gamma\gP^{\pi_*} \pp{Q^* - Q^{\mu_k}}
\end{align*}
By induction on $k$, we see that
\begin{align*}
    Q^*-Q^{\mu_{k+1}}
    \leq \frac{2\gamma}{1-\gamma} \sum_{j=0}^{k} \gamma^{k-j} \linf{\varepsilon_j}\const + \frac{\gamma (1+\gamma)}{1-\gamma} \sum_{j=0}^{k} \gamma^{k-j} \linf{\delta_{j+1}}\const + \pp{\gamma \gP^{\pi_*}}^{k+1} \pp{Q^* - Q^{\mu_0}}\,.
\end{align*}
By upper-bounding $Q^* - Q^{\mu_0}$ by $\linf{Q^* - Q^{\mu_0}}$, the claimed result is obtained.
\end{proof}

\section{A Proof of Theorem~\ref{theorem:error propagation of PQL with a fixed behavior policy} (PQL's Error Propagation with a Fixed Behavior Policy)}
\label{appendix:proof of PQL err prop with fixed behavior}

Here, we provide the error propagation analysis of PQL with a fixed behavior policy. While the behavior policy is not fixed in a practical situation, the error propagation analysis of PQL with a fixed behavior policy shows the trade-off between bias and convergence rate of PQL. This result is analogous to trade-offs explained in \citep{rowland2020apaptive} and sheds some light on a fundamental property of PQL.

\paragraph{Definition and Notation.}
We first recall our problem setting: (approximate) PQL updates its Q-function by
\begin{align*}
    \pi_k \in \greedy{Q_k}
    \text{ and }
    Q_{k+1} := \gN_{\lambda}^{\mu, \pi_k} Q_k + \varepsilon_k,
\end{align*}
We know that $\varepsilon_k (x, a) = 0$ guarantees the convergence of $Q_k$ to $Q^{\lambda \mu + (1-\lambda) \pi_\dagger}$, where $\pi_\dagger$ is $\argmax_{\pi} Q^{\lambda \mu + (1-\lambda) \pi}$. (See Section~\ref{section:analysis of PQL with a fixed behavior policy}.) Therefore, $\pi_K$ is an approximation of $\pi_\dagger$, and thus, it is natural to define a loss of using the policy $\pi_k$ rather than $\pi_\dagger$ by $V^{\rho_\dagger} - V^{\rho_K}$.

We define the following notations:
\begin{itemize}
    \item $\rho_\dagger := \lambda \mu + (1-\lambda) \pi_\dagger$
    \item $\rho_k := \lambda \mu + (1-\lambda) \pi_k$
    \item $d_k := Q^{\rho_\dagger} - Q_k$
    \item $b_k := Q_k - \gT^{\rho_k} Q_k$
    \item $\gA^{\dagger} := \gamma (1-\lambda) \inviplr{\mu} \gP^{\pi_\dagger}$
    \item $\gA_k := \gamma (1-\lambda) \gP^{\pi_k} \inviplr{\mu}$
    \item $\beta := \gamma (1-\lambda) / \pp{1 - \gamma \lambda}$
\end{itemize}
Note that $\gA_k$ is a contraction with respect to $L_\infty$-norm $\|\cdot\|_\infty$ with modulus $\beta$. (See Appendix~\ref{appendix:proof of different forms of PQL}.)

\paragraph{Proofs.}

Now we start proofs. Note that $V^{\rho_\dagger} - V^{\rho_K} \geq 0$. Indeed for any policy $\pi$ we have that $Q^{\rho_\dagger} - Q^{\lambda \mu + (1-\lambda)\pi} \geq 0$, and that $\pi_\dagger$ is greedy with respect to $Q^{\rho_\dagger}$. Accordingly $V^{\rho_\dagger} = \rho_\dagger Q^{\rho_\dagger} \geq (\lambda \mu + (1-\lambda)\pi) Q^{\rho_\dagger} \geq (\lambda \mu + (1-\lambda)\pi) Q^{\lambda \mu + (1-\lambda)\pi}$. The main strategy is the following: we first decompose $V^{\rho_\dagger} - V^{\rho_K} \geq 0$ to $V^{\rho_\dagger} - \rho_K Q_K$ and $\rho_K Q_K - V^{\rho_K} = \rho_K (Q_K - Q^{\rho_K})$; then we note that $V^{\rho_\dagger} - \rho_K Q_K \leq \rho_\dagger (Q^{\rho_\dagger} - Q_K)$ because of $\pi_K \in \greedy{Q_K}$ and $\rho_K = \lambda \mu + (1-\lambda) \pi_K$; these results tell us that we need upper bounds of $Q^{\rho_\dagger} - Q_K$ and $Q_K - Q^{\rho_K}$, which we shall derive.

We first prove an upper bound of $d_K = Q^{\rho_\dagger} - Q_K$.
\begin{lemma}\label{lemma:dk upper bound fixed behavior}
For any non-negative integer $K$, the following holds:
\begin{align*}
    d_K \leq \pp{\gA^\dagger}^K d_0 + \sum_{k=0}^{K-1} \pp{\gA^\dagger}^{K-k-1} \varepsilon_k\, .
\end{align*}
\end{lemma}

\begin{proof}
From Lemma~\ref{lemma:different forms of PQL}, we may deduce that
\begin{align*}
    d_K
    &= Q^{\rho_\dagger} - \inviplr{\mu} \pp{ r + \gamma (1-\lambda) \gP^{\pi_{K-1}} Q_{K-1} } - \varepsilon_{K-1}
    \\
    &= \inviplr{\mu} \bb{ Q^{\rho_\dagger} - \gamma \lambda \gP^{\mu} Q^{\rho_\dagger} - r - \gamma (1-\lambda) \gP^{\pi_{K-1}} Q_{K-1} } - \varepsilon_{K-1}
    \\
    &= \inviplr{\mu} \pp{ \gamma \gP^{\rho_\dagger} Q^{\rho_\dagger} - \gamma \lambda \gP^{\mu} Q^{\rho_\dagger} - \gamma (1-\lambda) \gP^{\pi_{K-1}} Q_{K-1} } - \varepsilon_{K-1}
    \\
    &= \gamma (1-\lambda) \inviplr{\mu} \pp{ \gP^{\pi_\dagger} Q^{\rho_\dagger} -  \gP^{\pi_{K-1}} Q_{K-1} } - \varepsilon_{K-1}\, ,
\end{align*}
where the last line follows from $\rho_\dagger = \lambda \mu + (1-\lambda) \pi_\dagger$. Because $\pi_{K-1} \in \greedy{Q_{K-1}}$, we have $\gP^{\pi_{K-1}} Q_{K-1} \geq \gP^{\pi_\dagger} Q_{K-1}\,$. Furthermore, since $\inviplr{\mu} = \sum_{t=0}^\infty \gamma^t \lambda^t (\gP^{\mu})^t$ is monotone, $\inviplr{\mu} \gP^{\pi_{K-1}} Q_{K-1} \geq \inviplr{\mu} \gP^{\pi_\dagger} Q_{K-1}$. As a result,
\begin{align*}
    d_K
    \leq \gamma (1-\lambda) \inviplr{\mu} \gP^{\pi_\dagger} \pp{ Q^{\rho_\dagger} - Q_{K-1} } - \varepsilon_{K-1}
    = \gA^{\dagger} d_{K-1} - \varepsilon_{K-1}\, .
\end{align*}
By induction on $K$, the claim is proven.
\end{proof}

We next prove an upper bound for $Q_K - Q^{\rho_K}$. To this end, note that
\begin{align*}
    Q_K - Q^{\rho_K} = \pp{\gI - \gamma \gP^{\rho_K}}^{-1} \pp{Q_K - \gT^{\rho_K} Q_K} = \pp{\gI - \gamma \gP^{\rho_K}}^{-1} b_K.
\end{align*}
Therefore, we need an upper bound for $b_K$, which is given below.
\begin{lemma}\label{lemma:bk upper bound fixed behavior}
For any non-negative integer $K$, the following holds:
\begin{align*}
    b_K
    &\leq \gA_{K-1} \cdots \gA_{0} b_0 + \sum_{k=0}^{K-1} \gA_{K-1} \cdots \gA_{k+1} \pp{\gI - \gamma \gP^{\rho_k}} \varepsilon_k\, ,
\end{align*}
where $\gA_{K-1} \cdots \gA_K := \gI$\,.
\end{lemma}

\begin{proof}
By a simple calculation, and $\pi_K \in \greedy{Q_K}$,
\begin{align*}
    b_K
    = Q_K - r - \gamma \lambda \gP^{\mu} Q_K - \gamma (1-\lambda) \gP^{\pi_K} Q_K
    \leq \iplr{\mu} Q_K - r - \gamma (1-\lambda) \gP^{\pi_{K-1}} Q_K\, .
\end{align*}
From Lemma~\ref{lemma:different forms of PQL}, we may deduce that
\begin{align*}
    b_K
    &\leq \gamma (1-\lambda) \gP^{\pi_{K-1}} Q_{K-1} + \iplr{\mu} \varepsilon_{K-1} - \gamma (1-\lambda) \gP^{\pi_{K-1}} Q_K
    \\
    &= \gamma (1-\lambda) \gP^{\pi_{K-1}} \pp{Q_{K-1} - Q_K} + \iplr{\mu} \varepsilon_{K-1}
    \\
    &= \gamma (1-\lambda) \gP^{\pi_{K-1}} \inviplr{\mu} \pp{Q_{K-1} - \gT^{\rho_{K-1}} Q_{K-1}} + \pp{\gI - \gamma \gP^{\rho_{K-1}}} \varepsilon_{K-1}
    \\
    &= \gA_{K-1} b_{K-1} + \pp{\gI - \gamma \gP^{\rho_{K-1}}} \varepsilon_{K-1}\, .
\end{align*}
By induction on $K$, the claim is proven.
\end{proof}

Now we are ready to prove an upper bound for $V^{\rho_\dagger} - V^{\rho_K}$. It is easy to derive the following two inequalities from the monotonicity of $\gA_k$ and $\gA^\dagger$:
\begin{align*}
  b_K
  &\leq \gA_{K-1} \cdots \gA_{0} \linf{b_0} \const + \pp{1 + \gamma} \sum_{k=0}^{K-1} \gA_{K-1} \cdots \gA_{k+1} \linf{\varepsilon_k} \const
  \\
  &= \beta^K \linf{b_0} \const + \pp{1 + \gamma} \sum_{k=0}^{K-1} \beta^{K-k-1} \linf{\varepsilon_k} \const\, ,
\end{align*}
and
\begin{align*}
    d_K
    &\leq \pp{\gA^\dagger}^K \linf{d_0} \const + \sum_{k=0}^{K-1} \pp{\gA^\dagger}^{K-k-1} \linf{\varepsilon_k} \const
    \\
    &= \beta^K \linf{d_0} \const + \sum_{k=0}^{K-1} \beta^{K-k-1} \linf{\varepsilon_k} \const\, .
\end{align*}
Note that
\begin{align*}
    V^{\rho_\dagger} - V^{\rho_K}
    &= \rho_\dagger Q^{\rho_\dagger} - \rho_K Q_K + \rho_K Q_K - V^{\rho_K}
    \\
    &\leq \rho_\dagger \pp{ Q^{\rho_\dagger} - Q_K } + \rho_K Q_K - V^{\rho_K}
    \\
    &= \rho_\dagger d_K + \rho_K \pp{\gI - \gamma \gP^{\rho_K}}^{-1} b_K.
\end{align*}
Therefore, we may deduce that
\begin{align*}
  V^{\rho_\dagger} - V^{\rho_K} \leq \beta^K \pp{ \linf{ d_0 } + \frac{\linf{b_0}}{1-\gamma} } \const + \frac{2}{1-\gamma} \sum_{k=0}^{K-1} \beta^{K-k-1} \linf{\varepsilon_k} \const,
\end{align*}
where we used $1 + (1+\gamma) / (1-\gamma) = 2 / (1-\gamma)$. Because $V^{\rho_\dagger} - V^{\rho_K} \geq 0$ and the right hand side is independent of a state,
\begin{align*}
  \linf{V^{\rho_\dagger} - V^{\rho_K}} \leq \beta^K \pp{ \linf{ d_0 } + \frac{\linf{b_0}}{1-\gamma} } + \frac{2}{1-\gamma} \sum_{k=0}^{K-1} \beta^{K-k-1} \linf{\varepsilon_k},
\end{align*}
This concludes the proof.

\section{A Proof of Theorem~\ref{theorem:error propagation of PQL with behavior policy updates} (PQL's Error Propagation with Behavior Policy Updates)}
\label{Appendix:proof of PQL err prop with behavior updates}

Here we provide error propagation analysis of PQL with behavior policy updates. Concretely, we derive the following bound:
\begin{align*}
  \linf{V^* - V^{\pi_K}} \leq \zeta^K \linf{Q^* - Q_0} + \frac{\zeta^K}{1-\gamma} \linf{b_0} + \sum_{l=0}^{K-1} \frac{2 \zeta^{K-l-1}}{1-\gamma} \linf{\varepsilon_l}\,,
\end{align*}
where $\zeta := 1-\alpha + \alpha \gamma$.

\paragraph{Definition and Notation.}
We first recall our problem setting: (approximate) PQL updates its Q-function by
\begin{align*}
    \pi_k \in \greedy{Q_k}\,,\,
    \mu_k = \alpha \pi_k + (1-\alpha) \mu_{k-1}
    \text{ and }
    Q_{k+1} := \gN_{\lambda}^{\mu_k, \pi_k} Q_k + \varepsilon_k,
\end{align*}
where $\mu_{-1}$ is arbitrary.

We define the following notations, some of which differ from those defined in Appendix~\ref{appendix:proof of PQL err prop with fixed behavior}:
\begin{itemize}
    \item $\rho_k := \lambda {\textcolor{red}{\mu_k}} + (1-\lambda) \pi_k$
    \item $b_k := Q_k - \gT^{\rho_k} Q_k$
    \item $d_k := Q^{\textcolor{red}{*}} - Q_k$
    \item $\gA_k := \gamma (1-\lambda) \gP^{\pi_k} \inviplr{\textcolor{red}{\mu_k}}$
    \item $\beta := \gamma (1-\lambda) / \pp{1 - \gamma \lambda}$
    \item $\gP^* := \gP^{\pi_*}$
    \item $\gT^* := \gT^{\pi_*}$
\end{itemize}
Here we highlighted (by red color texts) differences from the definitions in Appendix~\ref{appendix:proof of PQL err prop with fixed behavior}. Note that $\gA_k$ is still a contraction with modulus $\beta$.

\paragraph{Proofs.}
Now we start the proof. The main strategy is the almost same as the one we used in Appendix~\ref{appendix:proof of PQL err prop with fixed behavior}: we first decompose $V^* - V^{\pi_K}$ to two components $V^* - \pi_K Q_K$ and $\pi_K Q_K - V^{\pi_K}$, and then, we show an upper bound to each of them.

We first prove an upper bound for $b_k$, which turns out to be useful later.
\begin{lemma}\label{lemma:bk upper bound with behavior updates}
For any non-negative integer $k$, the following holds:
\begin{align*}
  b_k
  &\leq \gA_{k-1} \cdots \gA_{0} b_0 + \sum_{l=0}^{k-1} \gA_{k-1} \cdots \gA_{l+1} \pp{\gI - \gamma \gP^{\rho_l} } \varepsilon_l\,,
\end{align*}
where $\gA_{-1} \cdots \gA_{0} = \gI$\,.
\end{lemma}

\begin{proof}
Because $\pi_k \in \greedy{Q_k}$,
\begin{align*}
    \gT^{\rho_k} Q_k = \lambda \gT^{\mu_k} Q_k + (1-\lambda) \gT^{\pi_k} Q_k \geq \gT^{\mu_k} Q_k.
\end{align*}
Therefore, $b_k \leq Q_k - \gT^{\mu_k} Q_k$. By the assumption on $\mu_k$,
\begin{align*}
    \gT^{\mu_k} Q_k
    &= r + \gamma (1-\alpha) \gP^{\mu_{k-1}} Q_k + \gamma \alpha \gP^{\pi_k} Q_k
    \\
    &= r + \gamma \lambda \gP^{\mu_{k-1}} Q_k + \gamma (1-\lambda) \gP^{\pi_k} Q_k + \gamma \pp{\alpha - (1 - \lambda)} \pp{\gP^{\pi_k} Q_k - \gP^{\mu_{k-1}} Q_k}
    \\
    &\geq r + \gamma \lambda \gP^{\mu_{k-1}} Q_k + \gamma (1-\lambda) \gP^{\pi_k} Q_k
    \\
    &\geq r + \gamma \lambda \gP^{\mu_{k-1}} Q_k + \gamma (1-\lambda) \gP^{\pi_{k-1}} Q_k\,,
\end{align*}
where the third line follows since $\alpha \geq 1-\lambda$. Consequently
\begin{align*}
    b_k
    \leq Q_k - r - \gamma \lambda \gP^{\mu_{k-1}} Q_k - \gamma (1-\lambda) \gP^{\pi_{k-1}} Q_k
    = \iplr{\mu_{k-1}} Q_k - r - \gamma (1-\lambda) \gP^{\pi_{k-1}} Q_k\,.
\end{align*}
From Lemma~\ref{lemma:different forms of PQL}, we may deduce that
\begin{align*}
    b_k
    &\leq r + \gamma (1-\lambda) \gP^{\pi_{k-1}} Q_{k-1} + \pp{\gI - \gamma \lambda \gP^{\mu_{k-1}} } \varepsilon_{k-1} - r - \gamma (1-\lambda) \gP^{\pi_{k-1}} Q_k
    \\
    &= \gamma (1-\lambda) \gP^{\pi_{k-1}} \pp{ Q_{k-1} - Q_k } + \iplr{\mu_{k-1}} \varepsilon_{k-1}
    \\
    &= \gA_{k-1} \pp{ Q_{k-1} - \gT^{\rho_{k-1}} Q_{k-1} } - \gamma (1-\lambda) \gP^{\pi_{k-1}} \varepsilon_{k-1} + \iplr{\mu_{k-1}} \varepsilon_{k-1}
    \\
    &= \gA_{k-1} \underbrace{\pp{ Q_{k-1} - \gT^{\rho_{k-1}} Q_{k-1} }}_{=b_{k-1}} + \pp{\gI - \gamma \gP^{\rho_{k-1}} } \varepsilon_{k-1},
\end{align*}
where the last line follows from the definition of $\rho_{K-1} = \lambda \mu_{K-1} + (1-\lambda) \pi_{K-1}$. Therefore, by induction on $k$, we may deduce that
\begin{align*}
  b_k
  &\leq \gA_{k-1} \cdots \gA_{0} b_0 + \sum_{l=0}^{k-1} \gA_{k-1} \cdots \gA_{l+1} \pp{\gI - \gamma \gP^{\rho_l} } \varepsilon_l\, .
\end{align*}
This concludes the proof.
\end{proof}

We use a simple corollary of this lemma, derived based on the monotonicity of $(\gA_k)_{k\geq0}$ and $(\gP^{\rho_k})_{k\geq0}$.
\begin{corollary}\label{corollary:bk upper bound with behavior updates}
For any non-negative integer $k$, the following holds:
\begin{align*}
  b_k
  &\leq \beta^k \linf{b_0} \const + (1+\gamma)\sum_{l=0}^{k-1} \beta^{k-l-1} \linf{\varepsilon_l} \const := \bar{b}_k\, .
\end{align*}
\end{corollary}

We next prove an upper bound for $Q^* - Q_K$.
\begin{lemma}\label{lemma:dk upper bound with behavior updates}
For any non-negative integer $K$, the following holds:
\begin{align*}
  d_K
  &\leq \zeta^K \linf{Q^* - Q_0} + \sum_{l=0}^{K-1} \zeta^{K-1-l} \pp{ \frac{1-\alpha(1-\gamma\lambda)}{1-\gamma\lambda} \bar{b}_l + \linf{\varepsilon_l} \const }\, ,
\end{align*}
where $\zeta := 1-\alpha + \gamma \alpha$, and $\bar{b}_l$ is defined in Corollary~\ref{corollary:bk upper bound with behavior updates}.
\end{lemma}

\begin{proof}
We note that
\begin{align*}
    Q_K
    &= Q_{K-1} + \inviplr{\mu_{K-1}} \pp{ \gT^{\rho_{K-1}} Q_{K-1} - Q_{K-1} } + \varepsilon_{K-1}
    \\
    &= \gT^{\rho_{K-1}} Q_{K-1} - \gamma \lambda \gP^{\mu_{K-1}} \inviplr{\mu_{K-1}} b_{K-1} + \varepsilon_{K-1}\, .
\end{align*}
Let us focus on deriving a lower bound of $\gT^{\rho_{K-1}} Q_{K-1}$. From the definition of $\rho_{K-1}$ and $\mu_{K-1}$,
\begin{align*}
    \gT^{\rho_{K-1}} Q_{K-1}
    &= (1-\lambda) \gT^{\pi_{K-1}} Q_{K-1} + \lambda \gT^{\mu_{K-1}} Q_{K-1}
    \\
    &= (1-\lambda+\alpha\lambda) \gT^{\pi_{K-1}} Q_{K-1} + (1-\alpha)\lambda \gT^{\mu_{K-2}} Q_{K-1}
    \\
    &= (1-\lambda+\alpha\lambda) \gT^{\pi_{K-1}} Q_{K-1} - (1-\alpha) (1-\lambda) \gT^{\pi_{K-1}} Q_{K-1} + (1-\alpha) Q_{K-1} 
    \\
    &\hspace{9em}+ (1-\alpha)\bb{ \lambda \gT^{\mu_{K-2}} Q_{K-1} + (1-\lambda) \gT^{\pi_{K-1}} Q_{K-1} - Q_{K-1}}
    \\
    &= \alpha \gT^{\pi_{K-1}} Q_{K-1} + (1-\alpha) Q_{K-1} - (1-\alpha) \bb{ Q_{K-1} - \lambda \gT^{\mu_{K-2}} Q_{K-1} - (1-\lambda) \gT^{\pi_{K-1}} Q_{K-1}}\,.
\end{align*}
Recall that the first step of proving Lemma~\ref{lemma:bk upper bound with behavior updates} is showing that $b_k \leq Q_k - \lambda \gT^{\mu_{k-1}} Q_k - (1-\lambda) \gT^{\pi_k} Q_k$. Therefore the upper bound of $b_{K-1}$ in the lemma can serve as an upper bound of $Q_{K-1} - \lambda \gT^{\mu_{K-2}} Q_{K-1} - (1-\lambda) \gT^{\pi_{K-1}} Q_{K-1}$ too. Accordingly,
\begin{align*}
    Q^* - Q_K
    &\leq Q^* - \alpha \gT^{\pi_{K-1}} Q_{K-1} - (1-\alpha) Q_{K-1} + \frac{1-\alpha(1-\gamma\lambda)}{1-\gamma\lambda} \bar{b}_{K-1} + \linf{\varepsilon_{K-1}} \const
    \\
    &\leq \bb{(1-\alpha) \gI + \alpha \gamma \gP^*} \pp{Q^* - Q_{K-1}} + \frac{1-\alpha(1-\gamma\lambda)}{1-\gamma\lambda} \bar{b}_{K-1} + \linf{\varepsilon_{K-1}} \const\,.
\end{align*}
By induction on $K$, we deduce that
\begin{align*}
    Q^* - Q_K
    &\leq \bb{(1-\alpha) \gI + \alpha \gamma \gP^*}^K \pp{Q^* - Q_0} + \sum_{l=0}^{K-1} \zeta^{K-1-l} \pp{ \frac{1-\alpha(1-\gamma\lambda)}{1-\gamma\lambda} \bar{b}_l + \linf{\varepsilon_l} \const }
    \\
    &\leq \zeta^K \linf{Q^* - Q_0} \const + \sum_{l=0}^{K-1} \zeta^{K-1-l} \pp{ \frac{1-\alpha(1-\gamma\lambda)}{1-\gamma\lambda} \bar{b}_l + \linf{\varepsilon_l} \const }\,.
\end{align*}
This concludes the proof.
\end{proof}

Now we are ready to prove an upper bound for $V^* - V^{\pi_K}$. Note that from Corollary~\ref{corollary:bk upper bound with behavior updates} and Lemma~\ref{lemma:dk upper bound with behavior updates}
\begin{align*}
    V^* - V^{\pi_K}
    &= \pi_* Q^* - \pi_K Q_K + \pi_K Q_K - V^{\pi_K}
    \\
    &\leq \pi_* \pp{ Q^* - Q_K } + \pi_K Q_K - \pi_K Q^{\pi_K}
    \\
    &= \pi_* d_K + \pi_K \pp{\gI - \gamma \gP^{\pi_K}}^{-1} b_K
    \\
    &\leq \zeta^K \linf{Q^* - Q_0} \const + \sum_{l=0}^{K-1} \zeta^{K-1-l} \pp{ \frac{1-\alpha(1-\gamma\lambda)}{1-\gamma\lambda} \bar{b}_l + \linf{\varepsilon_l} \const } + \frac{1}{1-\gamma} \bar{b}_K\,.
\end{align*}
We simplify $\sum_{l=0}^{K-1} \zeta^{K-1-l} \bar{b}_l$ as follows:
\begin{align*}
    \sum_{l=0}^{K-1} \zeta^{K-1-l} \bar{b}_l
    &= \sum_{l=0}^{K-1} \zeta^{K-1-l} \beta^l \linf{b_0} \const + \sum_{l=0}^{K-1} \zeta^{K-l-1} (1+\gamma) \sum_{m=0}^{l-1} \beta^{l-m-1} \linf{\varepsilon_m} \const
    \\
    &= \frac{\zeta^K - \beta^K}{\zeta - \beta} \linf{b_0} \const + (1+\gamma) \sum_{m=0}^{K-2} \sum_{l=m+1}^{K-1} \zeta^{K-l-1} \beta^{l-m-1} \linf{\varepsilon_m} \const
    \\
    &= \frac{\zeta^K - \beta^K}{\zeta - \beta} \linf{b_0} \const + (1+\gamma) \sum_{m=0}^{K-2} \sum_{l=0}^{K-m-2} \zeta^{K-m-l-2} \beta^l \linf{\varepsilon_m} \const
    \\
    &= \frac{\zeta^K - \beta^K}{\zeta - \beta} \linf{b_0} \const + (1+\gamma) \sum_{m=0}^{K-2} \frac{\zeta^{K-m-1} - \beta^{K-m-1}}{\zeta - \beta} \linf{\varepsilon_m} \const\,,
\end{align*}
where the last line follows from
\begin{align*}
    \sum_{l=0}^{K-m-2} \zeta^{K-m-l-2} \beta^l
    = \zeta^{K-m-2} \sum_{l=0}^{K-m-2} \pp{\frac{\beta}{\zeta}}^l
    = \zeta^{K-m-2} \frac{1-\pp{\frac{\beta}{\zeta}}^{K-m-1}}{1-\frac{\beta}{\zeta}}
    = \frac{\zeta^{K-m-1} - \beta^{K-m-1}}{\zeta - \beta}\,.
\end{align*}
Using this result and
\begin{align*}
    \zeta - \beta
    = 1 - \alpha + \alpha \gamma - \frac{\gamma (1-\lambda)}{1-\gamma\lambda}
    = 1 - \frac{\gamma (1-\lambda)}{1-\gamma\lambda} - \alpha (1-\gamma)
    = \frac{1-\gamma}{1-\gamma\lambda} - \alpha (1-\gamma)
    = \frac{(1-\gamma) (1-\alpha(1-\gamma\lambda))}{1-\gamma\lambda}\,,
\end{align*}
we deduce that 
\begin{align*}
    V^* - V^{\pi_K}
    &\leq \zeta^K \linf{Q^* - Q_0} \const + \frac{\zeta^K - \beta^K}{1-\gamma} \linf{b_0} \const
    \\
    &\hspace{5em}+ (1+\gamma) \sum_{l=0}^{K-2} \frac{\zeta^{K-l-1} - \beta^{K-l-1}}{1-\gamma} \linf{\varepsilon_l} \const + \sum_{l=0}^{K-1} \zeta^{K-1-l} \linf{\varepsilon_l} \const + \frac{1}{1-\gamma} \bar{b}_K
    \\
    &= \zeta^K \linf{Q^* - Q_0}\const + \frac{\zeta^K}{1-\gamma} \linf{b_0} \const + \sum_{l=0}^{K-1} \frac{2 \zeta^{K-l-1}}{1-\gamma} \linf{\varepsilon_l}\const\,,
\end{align*}
where we used $1 + (1+\gamma) / (1-\gamma) = 2 / (1-\gamma)$. Because $V^* - V^{\pi_K} \geq 0$ and the right hand side is independent of a state,
\begin{align*}
  \linf{V^* - V^{\pi_K}} \leq \zeta^K \linf{Q^* - Q_0} + \frac{\zeta^K}{1-\gamma} \linf{b_0} + \sum_{l=0}^{K-1} \frac{2 \zeta^{K-l-1}}{1-\gamma} \linf{\varepsilon_l}.
\end{align*}
This concludes the proof.

\section{Details on Maximum-entropy RL}
\label{appendix:maxentrl}

The maximum-entropy RL \citep{ziebart2008maximum,fox2015taming,asadi2017alternative,haarnoja2017reinforcement,haarnoja2018soft} formulates that the agent maximizes both cumulative rewards and entropy at the same time. In particular, for a fixed $\alpha>0$, let $G_{\text{ent}}(x,a)$ be $\sum_{t=0}^\infty \gamma^t (R_t + \alpha H_t)$ conditional on $X_0=x,A_0=a$ where $H_t$ is the entropy of policy $\pi(\cdot|X_t)$. Define the maximum-entropy Q-function $Q_\text{ent}^\pi(x,a):= \E \bb{r(X_0,A_0) + \gamma G_\text{ent}(X_1,A_1) \mb X_0=x,A_0=a}$. It is then possible to define Bellman operators as well as their multi-step variants as in Section~\ref{section:operators}. Due to space limit, we postpone their details in Appendix~\ref{appendix:maxentrl}.

It is straightforward to extend off-policy Q($\lambda$) actor-critic algorithm to the formulation of maximum-entropy RL \citep{fox2015taming,haarnoja2017reinforcement}. Maximum-entropy actor-critic algorithms also maintain a Q-function $Q_\phi(x,a)$ along with a stochastic policy $\pi_\theta(a|x)$. With off-policy data $(x_t,a_t,r_t)_{t=0}^\infty$, one could modify Equation~\ref{eq:recursive} to recursively compute the Q-function targets as 
\begin{align}
    \hat{Q}_i = r_i + \gamma \hat{V}_\text{ent}(x_{i+1}) +  + \gamma\lambda\left(\hat{Q}_{i+1} - \hat{V}_\text{ent}(x_{i+1}) \right),
    \label{eq:maxent-recursive}
\end{align}
where the value target $\hat{V}_\text{ent}(x_{i+1})=Q_{\phi^-}(x_{i+1},\pi_{\theta^-}(x_{i+1})) + \alpha_{\text{td}} H(\pi_{\theta^-}(\cdot|x_{i+1}))$. Contrasting Equation~\ref{eq:maxent-recursive} and Equation~\ref{eq:recursive}, the major difference is that the Q-function target is augmented with an entropy bonus $\alpha_{\text{td}} H(\pi_{\theta^-}(\cdot|x_{i+1}))$. Given a batch of data $(x_0^{(j)},a_0^{(j))})_{j=1}^B$, The policy is updated via gradient ascent $\theta\leftarrow\theta + \nabla_\theta \frac{1}{B}\sum_{j=1}^V Q_\phi(x_0^{(j)},\pi_\theta(x_0^{(j)}) + \alpha_\text{pol} H(\pi_\theta(\cdot|x_0^{(j)}))$. See Appendix~\ref{appendix:maxentrl} for the pseudocode of the full algorithm.

In theory, here, one should set $\alpha_\text{pol}= \alpha_{\text{td}}=\alpha$ to ensure that the fixed point is unbiased when the collected data are on-policy $\mu=\pi$. However, in practice, we find that large $\alpha_\text{td}$ tends to destabilize the update. In particular, when setting $\alpha_\text{pol}=\alpha_\text{td}=0.1$ chosen as the default hyper-parameter, multi-step SAC does not learn stably. We hypothesize that this is because when $\alpha_\text{td}>0$, an entropy bonus term is added to the target Q-function at each step (over $n\geq 1$ steps), whose numerical scale makes it much more difficult to learn a proper Q-function. 

Instead, we find that a stable alternative is to set $\alpha_\text{td}=0$ except at the last time step, where $\alpha_\text{td}=\alpha=0.2$. This greatly stablizes the update as the intermediate entropy bonus is effectively removed. It is of interest to study how such bonus term affects the performance of multi-step algorithms and how to align the practice more consistently with theory.

\section{Experiments}
\label{appendix:experiment}
 
\subsection{Further details on implementations of Peng's Q($\lambda$)}

\paragraph{Generic off-policy actor-critic deep RL algorithms.} We provide pseudocode for generic off-policy actor-critic deep RL algorithms in Algorithm 1. These algorithms maintain a Q-function critic $Q_\phi(x,a)$ and a policy $\pi_\theta(x)$. In general, The algorithm collects data with an exploratory behavior policy $\mu$ and saves tuples $(x_t,a_t,r_t)$ into a replay buffer $\mathcal{D}$. At each training iteration, the critic $Q_\phi(x,a)$ is updated by minimizing squared errors against a Q-function target $\mathbb{E}_D\left\lbrack (Q_\phi(x,a)-Q_\text{target}(x,a))^2 \right\rbrack$. The policy is  updated via the deterministic policy gradient $\theta\leftarrow\theta + \alpha \mathbb{E}_\mu \left\lbrack \nabla_\theta  Q_\phi(x,\pi_\theta(x)) \right\rbrack$ \citep{silver2014deterministic}. 

Now, we focus on the definition of targets $Q_\text{target}(x,a)$. Given the transitions $(x,a,r,x')$, one popular choice (see, e.g., \citep{lillicrap2015continuous,fujimoto2018addressing}) is to compute the target as $Q_\text{target}(x,a) = r + \gamma Q_{\phi^-}(x',\pi_{\theta^-}(x'))$ where $\theta^-,\phi^-$ are delayed copies of $\theta,\phi$ respectively \citep{mnih2015human}. An interpretation is that since the policy follows the deterministic gradient through $Q_\phi(x,a)$, it serves as an approximate greedy operator $\pi_\theta(x)\approx \arg\max_a Q_\phi(x,a)$. Note that when $\A$ is continuous, the exact greedy operation $\max_a Q_\phi(x,a)$ is not tractable. In this sense, the above update is an approximate stochastic estimate of the Bellman operator $\mathcal{T}Q(x,a)$.

\begin{algorithm}[h]
\begin{algorithmic}
\REQUIRE policy $\pi_\theta(x)$, critic $Q_\phi(x,a)$, target parameters $\theta^-,\phi^-$ and learning rate $\alpha$ \\
\WHILE{not converged}
\STATE 1. Collect partial trajectories $(x_t,a_t,r_t)_{t=1}^T$ under behavior policy $\mu$.
\STATE 2. Samples $B$ partial trajectories each of length $n$ from the replay buffer $\mathcal{D}$. 
\STATE 3. Construct Q($\lambda$) targets $Q_\text{targ}^{(j)}$. Gradient descent update on critic $\phi\leftarrow\phi-\alpha \frac{1}{B}\nabla_\phi \sum_{j=1}^B (Q_\phi(x_0^{(j)},a_0^{(j)})-Q_\text{targ}^{(j)})^2$.
\STATE 4. Gradient ascent on policy $\theta\leftarrow\theta + \nabla_\theta \frac{1}{B}\sum_{j=1}^B Q_\phi(x_0^{(j)},\pi_\theta(x_0^{(j)}))$.
\STATE 5. Update the target parameters $\theta^-\leftarrow\theta,\phi^-\leftarrow\phi$.
\ENDWHILE
\caption{Off-policy Q($\lambda$) actor-critic algorithm}
\end{algorithmic}
\end{algorithm}

\paragraph{Recursive computations of Q-function targets.}
The target value defined by the Q($\lambda$) operator could be computed recursively. In particular, given an infinite trajectory $(x_0,a_0,r_0,x_1,a_1,r_1,...)$. Assume that we have a Q-function critic $Q_\phi(x,a)$. Let $\hat{Q}_i$ be the target value estimate at time step $i$, then
\begin{align*}
    \hat{Q}_i = r_i + \gamma \max_a Q_\phi(x_i,a) + \gamma \lambda\left(\hat{Q}_{i+1} - \max_a Q_\phi(x_i,a)\right).
\end{align*}
For continuous action space where computing $\max_a Q(x_i,a)$ is difficult, we propose to replace $\max_a Q_\phi(x,a)\approx Q_\phi(x,\pi_\theta(x))$. In addition, in practice, it is not feasible to generate trajectories of an infinite length. For a partial trajectory $(x_0,a_0,r_0,x_1,a_1,r_1,...x_n)$ of length $n$, we bootstrap the Q-function value at the end of the trajectory as $\hat{Q}_n=Q_{\phi^-}(x_i,\pi_{\theta(x_i)^-})$. Then the target at $(x_0,a_0)$ can be recursively computed as
\begin{align}
    \hat{Q}_i = r_i + \gamma Q_{\phi^-}(x_{i+1},\pi_{\theta^-}(x_{i+1})) + \gamma\lambda \left(\hat{Q}_{i+1} - Q_{\phi^-}(x_{i+1},\pi_{\theta^-}(x_{i+1}))\right).
    \label{eq:recursive}
\end{align}

\subsection{Implementations and algorithms for continuous control in deep RL}

\paragraph{Implementation code base.} We adapt the base implementations in OpenAI SpinningUp \citep{SpinningUp2018}. All algorithmic variants adopt default hyper-parameters from the code base. These include learning rates, batch size, replay buffer size,  target network update rules, as well as other missing hyper-parameters.

\paragraph{Deep deterministic policy gradient (DDPG).} DDPG \citep{lillicrap2015continuous} maintains a deterministic policy network $\pi_\theta(a|x)\equiv \pi_\theta(x)$ and a Q-function critic $Q_\phi(x,a)$. The algorithm explores by executing a perturbed policy $a = \epsilon + \pi_\theta(x)$ where $\epsilon\sim\mathcal{N}(0,\sigma^2)$ for $\sigma=0.1$, and then saves the data $(x,a,r,x^\prime)$ into a replay buffer $\mathcal{D}$. At training time, the behavior data is sampled uniformly from the replay buffer $(x_i,a_i,r_i,x_i^\prime)_{i=0}^{B-1}\sim\mathcal{U}(\mathcal{D})$ with $B=100$. The critic is updated via TD($0$), by minimizing: $\frac{1}{B}\sum_{i=0}^{B-1}(Q_\phi(x_i,a_i)-Q_{\text{target}}(x_i,a_i))^2$ where $Q_\text{target}(x_i,a_i) = r_i + \gamma Q_{\phi^\prime}(x_i^\prime,\pi_{\theta^\prime}(x_i^\prime))$, where $\theta^\prime,\phi^\prime$ are delayed versions of $\theta,\phi$ respectively \citep{mnih2015human}. The policy is updated by maximizing $\frac{1}{B}\sum_{i=0}^{B-1} Q_\phi(x_i,\pi_\theta(x_i))$ with respect to $\theta$. Both parameters $\theta,\phi$ are trained with the Adam optimizer \citep{kingma2014adam} with learning rate $\alpha=10^{-4}$. We adopt other default hyper-parameters in \citep{SpinningUp2018}, for details, please refer to the code base.

\paragraph{Twin-delayed deep deterministic policy gradient (TD3).} TD3 \citep{fujimoto2018addressing} adopts the same training pipeline and architectures as DDPG. TD3 also adopts two critic networks $Q_{\phi_1}(x,a),Q_{\phi_2}(x,a)$ with parameters $\phi_1,\phi_2$, in order to minimize the over-estimation bias \citep{hasselt2010double}.

\paragraph{Soft actor-critic (SAC).} SAC \citep{haarnoja2018soft} adopts the same training pipeline and architecture as DDPG and TD3. However, the critical difference is that SAC augments the reward functions with state-wise entropy to discourage the policy from collapsing to a deterministic distribution. It also maintains two networks to counter the over-estimation bias as TD3. Please see Appendix~\ref{appendix:maxentrl} for further backgrounds regarding maximum-entropy RL.

\subsection{Further details on baseline operators (algorithms)}

\paragraph{Uncorrected $n$-step.} We implement uncorrected $n$-step as one of the baseline algorithms \citep{hessel2017rainbow}. This implements the target Q-functions as $\hat{Q}_i=\sum_{j=i}^{i+n-1} \gamma^{j-i}r_j + \gamma^n \max_a Q_\phi(x_{i+n}, a)$ where $Q_\phi$ is the Q-function network. It is \emph{uncorrected} because there is no importance sampling ratios that adjust the discrepancy between the $\pi$ and $\mu$. In continuous control, the maximization operation is replaced by the output of the policy network, i.e. $ Q_{\phi^-}(x_{i+n}, \pi_{\theta^-}(x_{i+n})$. When $n=1$, we recover the one-step baseline of a vanilla baseline algorithm.

\paragraph{Peng's Q($\lambda$).} As briefly discussed in the main paper, we implement a version Peng's Q($\lambda$) with finite horizon $n$. This means that the recursive computation of target defined in Eqn~\ref{eq:recursive} holds until the $n$-th step, where $\hat{Q}_{i+n}=Q_{\phi^-}(x_{i+n},\pi_\theta(x_{i+n}))$. This is because in practice, trajectories are always truncated and of finite lengths, which implies that at the end of trajectories we need to bootstrap directly from the learned Q-functions. 

\paragraph{Retrace.} We implement Retrace \citep{munos2016safe} as a baseline algorithm for comparison. Retrace computes the Q-function target recursively as 
\begin{align}
    \hat{Q}_i = r_i + \gamma Q_{\phi^-}(x_{i+1},\pi_{\theta^-}(x_{i+1})) +  \gamma c_i \left( \hat{Q}_{i+1} -  Q_{\phi^-}(x_{i+1},a_{i+1})  \right).
    \label{eq:recursive-retrace}
\end{align}
Here, the trace coefficient $c_i=\lambda \min(\frac{\pi_\theta(a_i|x_i)}{\mu(a_i|x_i)},\bar{c})$ where $\bar{c}$ is the truncation level. By default, $\lambda=\bar{c}=1$. The motivation is that the variance is controlled by truncating the importance sampling ratio. As a result of the update, TD3 is not directly compatible with the update because it requires $\pi,\mu$ to be both stochastic. We implement a version of TD3 with a stochastic actor: $\pi_\theta(a|x)=\text{tanh}\left(\mu_\theta(x)+\sigma_\theta(x)\cdot\epsilon\right)$, where $\epsilon\sim\mathcal{N}(0,\mathbb{I})$ and $\text{tanh}(x)=(\exp(x)-\exp(-x))/(\exp(x)+\exp(-x)) \in (-1,1)$. The log probability $\log \pi(a|x)$ is still tractable and can be analytically computed (see, e.g., similar computations in \citep{haarnoja2018soft}). The behavior policy $\mu$ is implemented as $\mu(a|x) =\text{tanh}\left(\mu_\theta(x)+\sigma\cdot\epsilon\right)$ with a fixed standard deviation parameter $\sigma=0.1$. These hyper-parameters are chosen such that they match the scale of action perturbation in the original TD3 implementation.

\paragraph{Ctrace.} Ctrace \citep{rowland2020apaptive} is an adaptive off-policy learning algorithm based on Retrace. Its main idea is to adjust the target policy at evaluation time. Instead of evaluating $Q^\pi$, the target Q-function is changed to $Q^{\alpha\pi + (1-\alpha)\mu}$ where $\alpha\in[0,1]$ is a trainable coefficient that interpolates target policy and behavior policy. By changing $\alpha$, Ctrace achieves a trade-off between fixed point bias (against $Q^\pi$) and contraction rate. We always adapt $\alpha$ such that the contraction rate of the overall operator matches a particular value $\Gamma$. Since we implement a version of Ctrace with finite horizon $n$, we use the following modified definition of the contraction rate so that the contraction rate ranges from $0$ to $1$ regardless of $n$: $1 - \frac{1-\gamma}{1-\gamma^n} \E [ \sum_{t=0}^{n-1} \gamma^t \prod_{s=1}^t ( (1-\alpha)+\alpha \rho_s ) ]$, where $\rho_s := \pi_\theta (a_s | x_s) / \mu (a_s | x_s)$. Throughout experiments, we set $\Gamma=0.7$. See \citep{rowland2020apaptive} for more comprehensive description of the algorithm.

\paragraph{Tree-backup.} Similar to Retrace, algorithms such as tree-backup \citep{precup2000eligibility} also preserve the unbiased fixed point of the operator as $Q^\pi$. Tree-backup adopts the same recursive computation as Retrace in Eqn~\ref{eq:recursive-retrace} except that the trace coefficient is $c_i=\pi_\theta(a_i|x_i)$. However, the tree-backup algorithm was developed for discrete action space alone, where the probability $\pi_\theta(a_i|x_i)\in[0,1]$. For continuous control tasks, this is not true because $\pi_\theta(a|x)$ is a density. We observe that naive implementations of tree-backup algorithm leads to very unstable update because of the numerical scale of $\log \pi_\theta(a|x)$. Empirically, we find that the performance of tree-backup to be very poor on continuous control 
tasks and we do not include the results.

\subsection{Further details on the toy example}
At each iteration $t$ of the algorithm, we maintain a Q-function table $Q^{(t)}(x,a)$. Given a sampled trajectory $(x_t,a_t,r_t)_{t=0}^{D-1}$, the operator (e.g. Retrace or Peng's Q($\lambda$)) constructs targets $Q_\text{target}(x,a)$. The Q-functions are updated as $Q^{(t+1)}(x,a)\leftarrow(1-\alpha)Q^{(t)}(x,a)+\alpha Q_\text{target}(x,a)$. Then the policy is updated as $\pi^{(t)}\leftarrow (1-\alpha)\pi + \alpha \pi_g(Q^{(t)}(x,a))$ where $\pi_g(Q^{(t)}(x,a))$ is the greedy policy with respect to $Q^{(t)}(x,a)$. Throughout experiments, the learning rate is fixed $\alpha=0.1$.

When computing the target Q-functions $Q_\text{target}(x,a)$, we apply the recursive computations introduced in previous sections. This is applied to all state-action pairs along sampled trajectories. At each iteration, the algorithm collects $N=1$ trajectory from the MDP.

\subsection{Additional evaluations on standard benchmarks} 

\paragraph{Detailed hyper-parameters.} In the main paper, we use $n=5$ for all multi-step algorithms to cap the length of the partial trajectories. For Peng's Q($\lambda$), we set $\lambda=0.9$ throughout the experiments.

\paragraph{Further results.} See Figure~\ref{fig:standard-td3-appendix} for additional experiments on evaluations over standard benchmarks. We further evaluate TD3 variants over tasks from Bullet physics (B) and OpenAI gym (G). Throughout the experiments, we use $n=5$ for all multi-step algorithms to cap the length of the partial trajectories. For Peng's Q($\lambda$), we set $\lambda=0.7$. Overall, Peng's Q($\lambda$) performs fairly stably, though it does not perferm the best per task. Interestingly, Retrace performs fairly well on Ant(G), which is in sharp contrast to its relatively poor performance across other tasks. We no longer include DDPG as a baseline as it is generally considered a slightly less competitive baseline compared to TD3.

\begin{figure}[h]
    \centering
    \subfigure[Ant(G)]{\includegraphics[keepaspectratio,width=.23\textwidth]{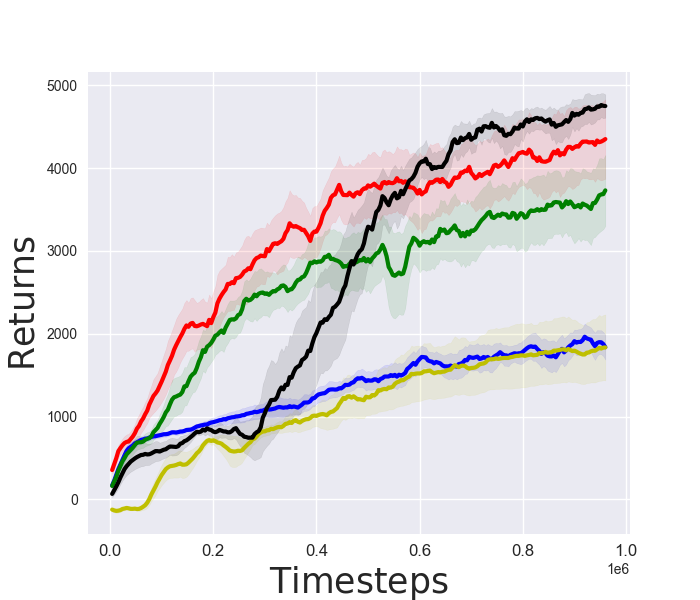}}
    \subfigure[Walker(G)]{\includegraphics[keepaspectratio,width=.23\textwidth]{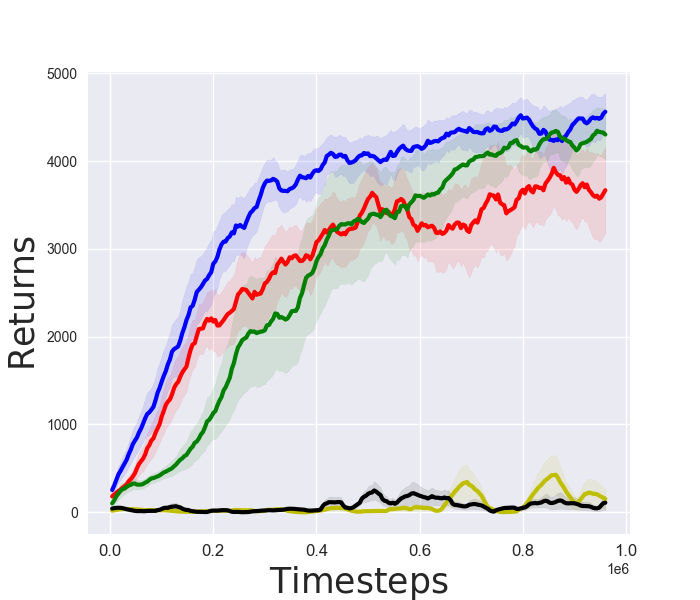}}
    \subfigure[Ant(B)]{\includegraphics[keepaspectratio,width=.23\textwidth]{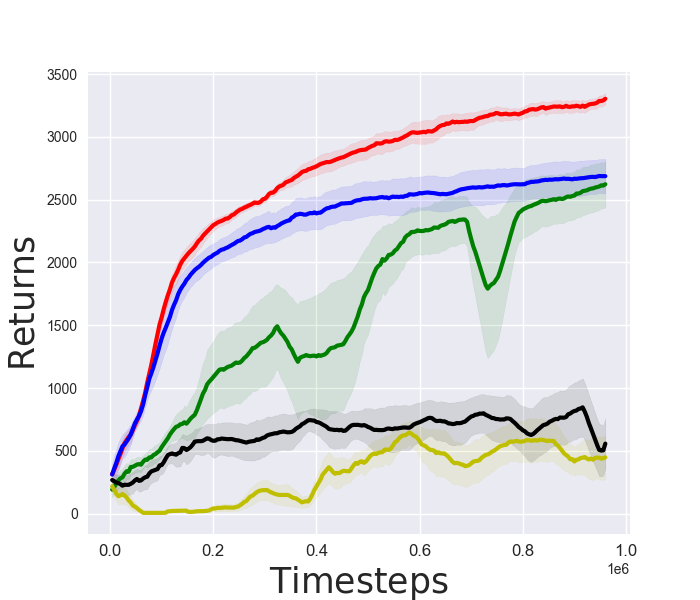}}
    \subfigure[HalfCheetah(B)]{\includegraphics[keepaspectratio,width=.23\textwidth]{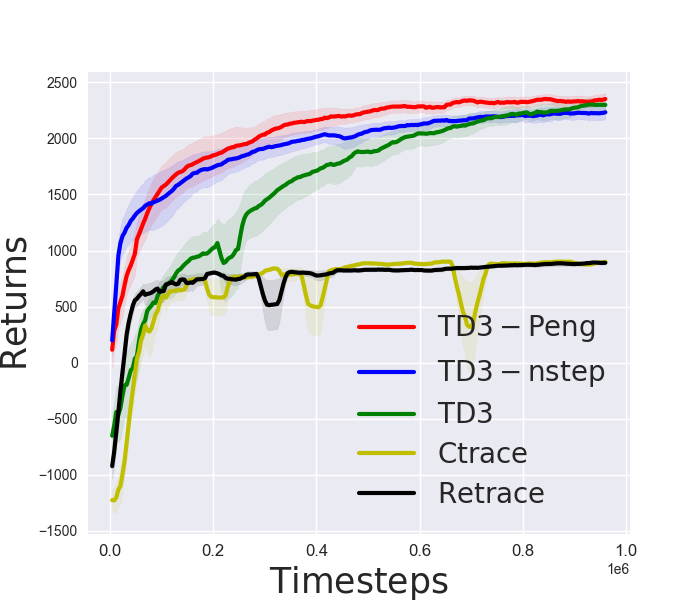}}
    \caption{Evaluation of TD3 baselines over continuous control domains. Each curve corresponds to a baseline algorithm averaged over 5 random seeds. (B) denotes tasks from Bullet physics and (G) denotes tasks from OpenAI gym.}
    \label{fig:standard-td3-appendix}
\end{figure}

\subsection{Additional evaluations on sparse rewards benchmarks} 

\paragraph{Sparse rewards.} We implement  delayed rewards as a form of sparse rewards. Delayed reward environment tests algorithms' capability to tackle delayed feedback in the form of sparse rewards \citep{oh2018self}. In particular, a standard benchmark environment returns dense reward $r_t$ at each step $t$. Consider accumulating the reward over $d$ consecutive steps and return the sum at the end $k$ steps, i.e. $r_t^\prime=0$ if $ t \ \text{mod}\ k \neq 0$ and $r_t^\prime = \sum_{\tau=t-d+1}^t r_\tau$ if $t\ \text{mod}\ d =0$. Throughout the experiments, we set $d=3$.

\paragraph{Detailed hyper-parameters.}We use $n=5$ for all multi-step algorithms to cap the length of the partial trajectories. For Peng's Q($\lambda$), we set $\lambda=0.7$ throughout the experiments.

\paragraph{Further results.} See Figure~\ref{fig:sparse-td3-appendix} for additional experiments on evaluations over standard benchmarks. We further evaluate TD3 variants over tasks from Bullet physics (B) and OpenAI gym (G). Throughout the experiments, we use $n=5$ for all multi-step algorithms to cap the length of the partial trajectories. For Peng's Q($\lambda$), we set $\lambda=0.7$. Overall, Peng's Q($\lambda$) performs fairly stably, though it does not perferm the best per task. Interestingly, consistent with results in Figure~\ref{fig:standard-td3-appendix}, Retrace performs well in Ant(G) with sparse rewards.

\begin{figure}[h]
    \centering
    \subfigure[Ant(G)]{\includegraphics[keepaspectratio,width=.23\textwidth]{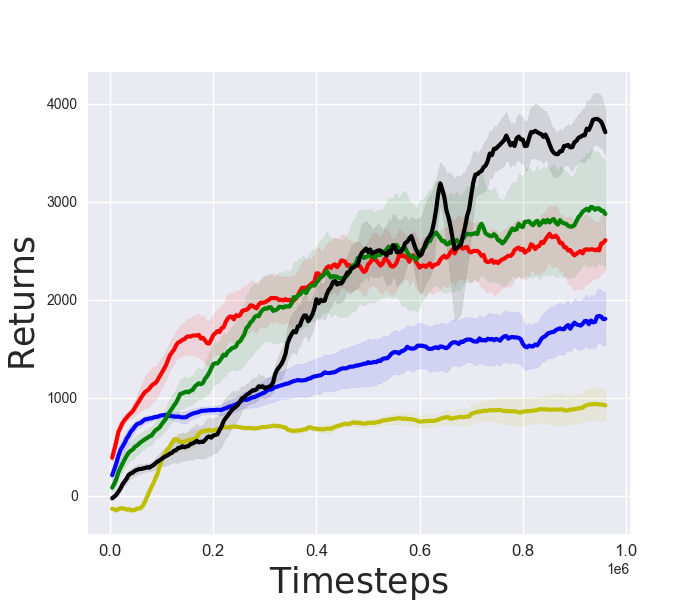}}
    \subfigure[Walker(G)]{\includegraphics[keepaspectratio,width=.23\textwidth]{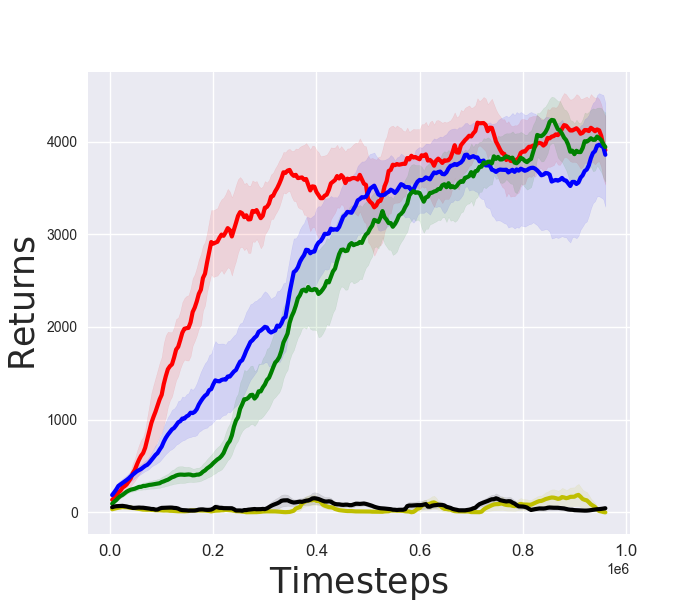}}
    \subfigure[Ant(B)]{\includegraphics[keepaspectratio,width=.23\textwidth]{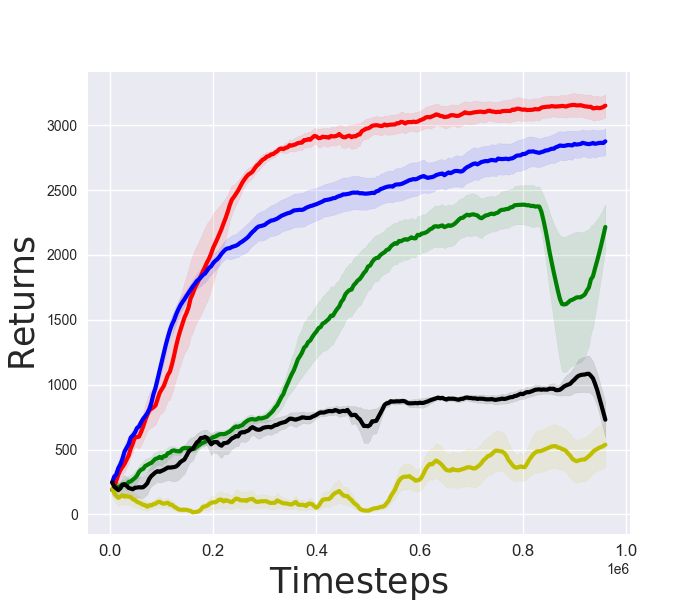}}
    \subfigure[HalfCheetah(B)]{\includegraphics[keepaspectratio,width=.23\textwidth]{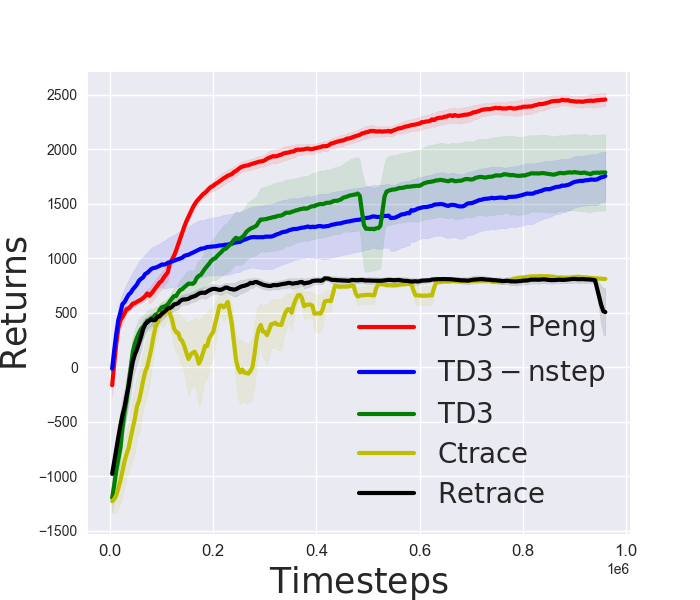}}
    \caption{Evaluation of TD3 baselines over continuous control domains with sparse rewards. Each curve corresponds to a baseline algorithm averaged over 5 random seeds. (B) denotes tasks from Bullet physics and (G) denotes tasks from OpenAI gym.}
    \label{fig:sparse-td3-appendix}
\end{figure}

\subsection{Experiment results on maximum-entropy RL}

\begin{figure}[h]
    \centering
    \subfigure[Ant(G)]{\includegraphics[keepaspectratio,width=.22\textwidth]{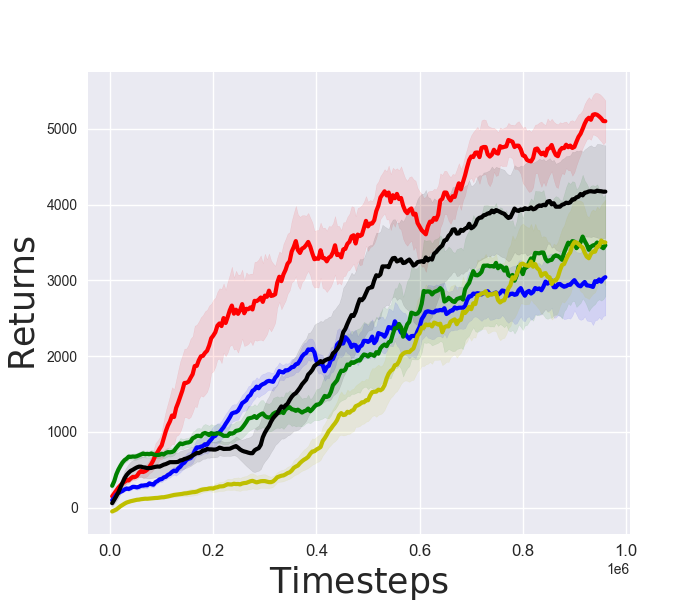}}
    \subfigure[Walker2d(G)]{\includegraphics[keepaspectratio,width=.22\textwidth]{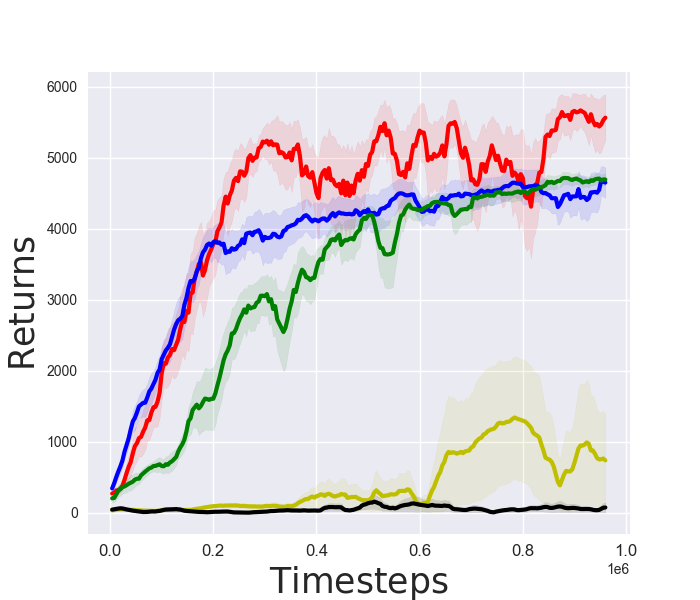}}
    \subfigure[Ant(B)]{\includegraphics[keepaspectratio,width=.22\textwidth]{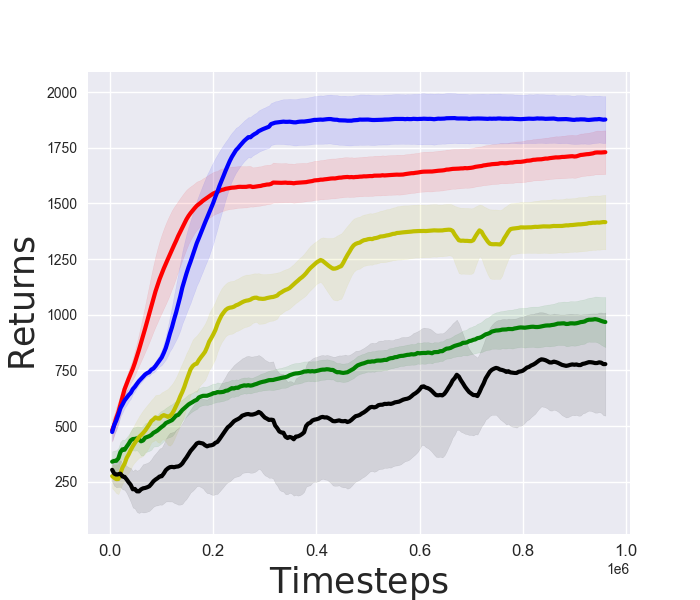}}
    \subfigure[HalfCheetah(B)]{\includegraphics[keepaspectratio,width=.22\textwidth]{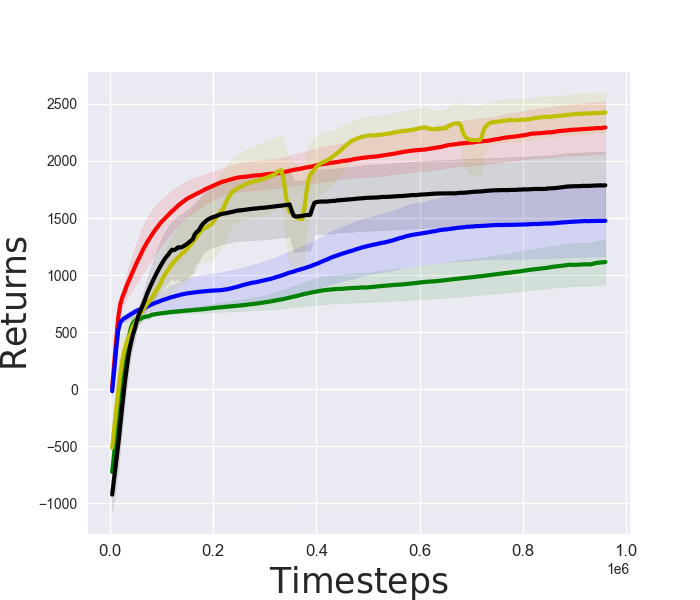}}
    \subfigure[CheetahRun(D)]{\includegraphics[keepaspectratio,width=.22\textwidth]{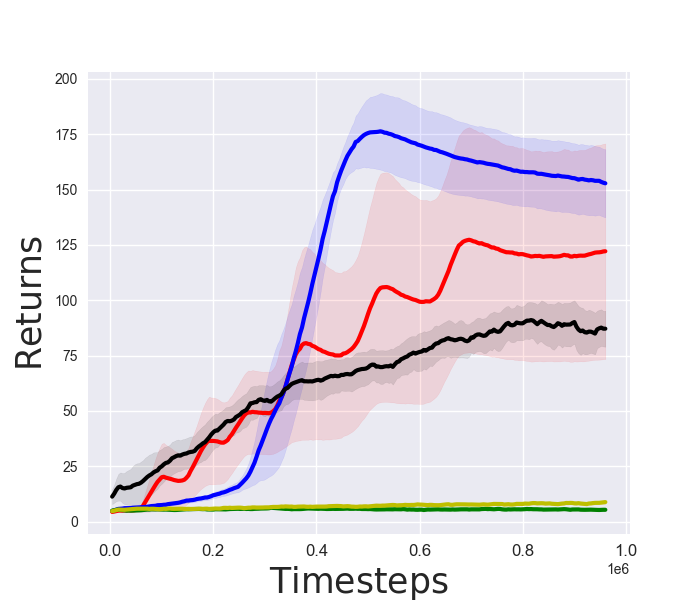}}
    \subfigure[WalkerStand(D)]{\includegraphics[keepaspectratio,width=.22\textwidth]{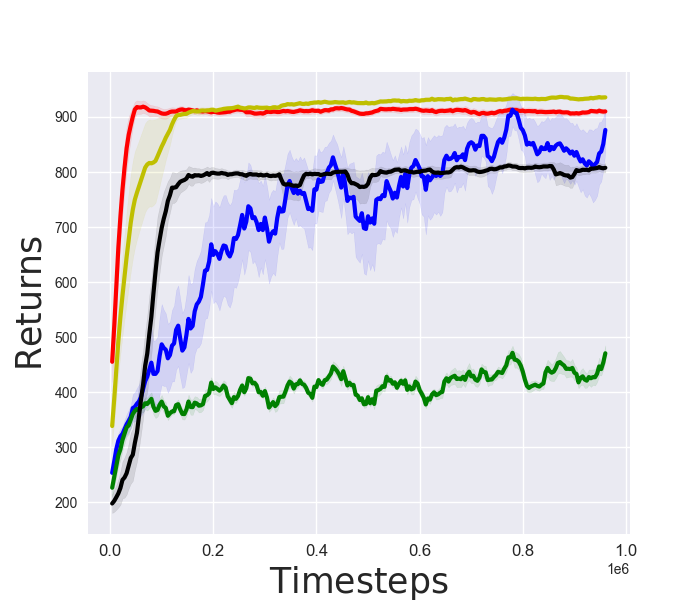}}
    \subfigure[WalkerRun(D)]{\includegraphics[keepaspectratio,width=.22\textwidth]{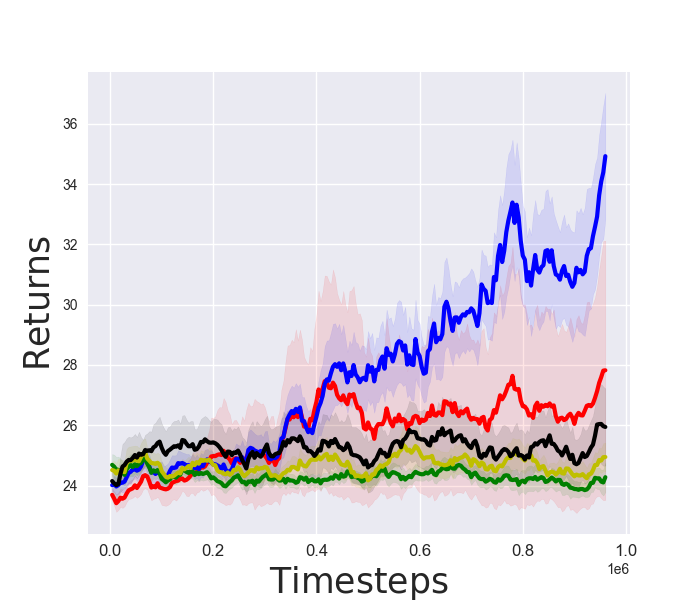}}
    \subfigure[WalkerWalk(D)]{\includegraphics[keepaspectratio,width=.22\textwidth]{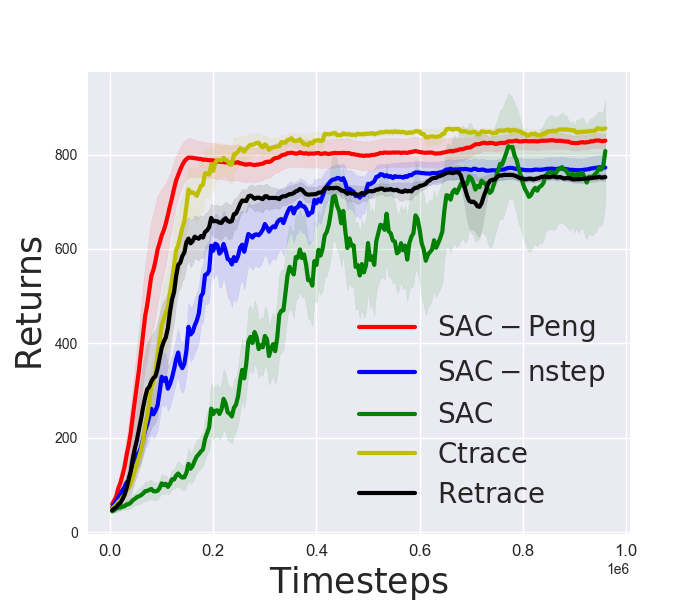}}
    \caption{Evaluation of soft actor-critic (SAC) variants over standard continuous control domains. Each curve corresponds to a baseline algorithm averaged over 5 random seeds. We consider tasks from gym (G), bullet physics (B) and DM control suite (D).}
    \label{fig:sac}
\end{figure}

We build on soft actor-critic (SAC) \citep{haarnoja2018soft} and evaluate algorithmic variants over standard benchmark tasks. For Peng's Q($\lambda$), we use $\lambda=0.7$. In Figure~\ref{fig:sac} we show the results across all selected benchmark tasks. Peng's Q($\lambda$) generally performs more stably than other baselien variants. This is highlighted by the fact that Peng's Q($\lambda$) always ranks as the top two baselines per each task. As an additional empirical observation, we find that SAC generally performs not as well as TD3 on DM control suites. We speculate that this might be because throughout the experiments we use $\alpha=0.2$. An adaptive entropy coefficient might further improve the performance.

\subsection{Ablation on $\lambda$}

In Figure~\ref{fig:lambda}, we show the ablation study on the sensitivity of Peng's Q($\lambda$) to its only hyper-parameter $\lambda$. We choose $\lambda \in \{0.3,0.5,0.7,0.9\}$ and examine the performance of the resulting algorithms over DM control suite (sparse rewards). Overall, we see that the best hyper-parameter is achieved $\lambda\approx 0.7$. When $\lambda$ deviates from this value, its performance is still relatively robust. When $\lambda$ decreases, we see its performance degrades more drastically than when it increases. Finally, it is worth noting that across all our previous evaluations, we always select $\lambda \in \{0.7,0.9\}$ and adopt a single $\lambda$ for benchmark tasks with the same simulation backend. 
This shows the robustness of Peng's Q($\lambda$) in practical applications.

\begin{figure}[h]
    \centering
    \subfigure[CheetahRun(D)]{\includegraphics[keepaspectratio,width=.22\textwidth]{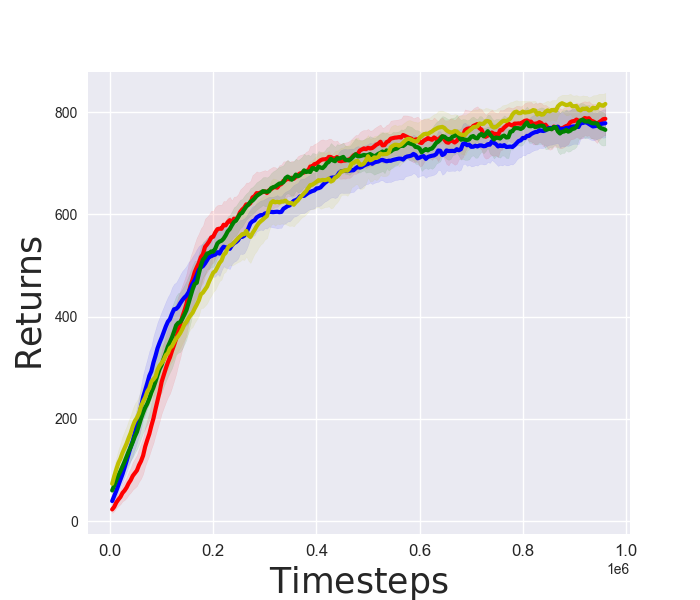}}
    \subfigure[WalkerStand(D)]{\includegraphics[keepaspectratio,width=.22\textwidth]{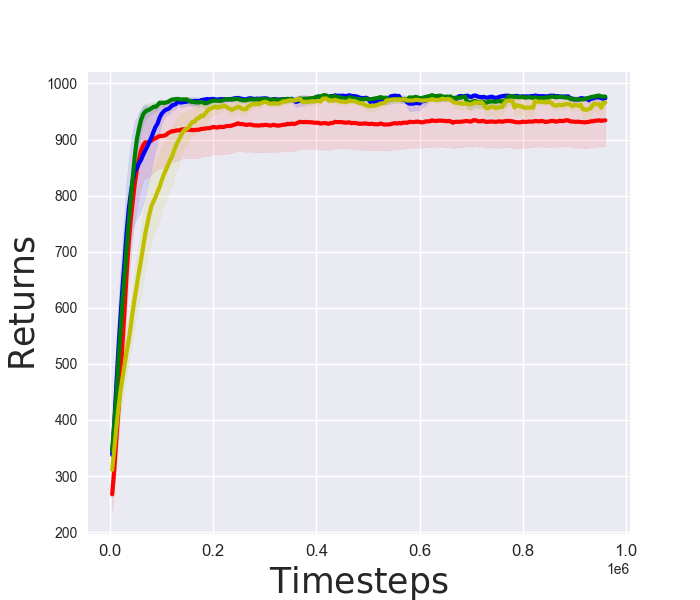}}
    \subfigure[WalkerRun(D)]{\includegraphics[keepaspectratio,width=.22\textwidth]{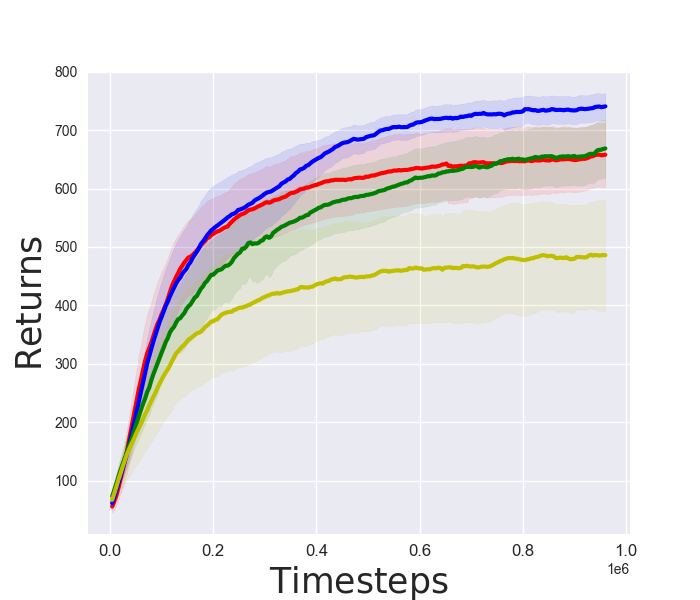}}
    \subfigure[WalkerStand(D)]{\includegraphics[keepaspectratio,width=.22\textwidth]{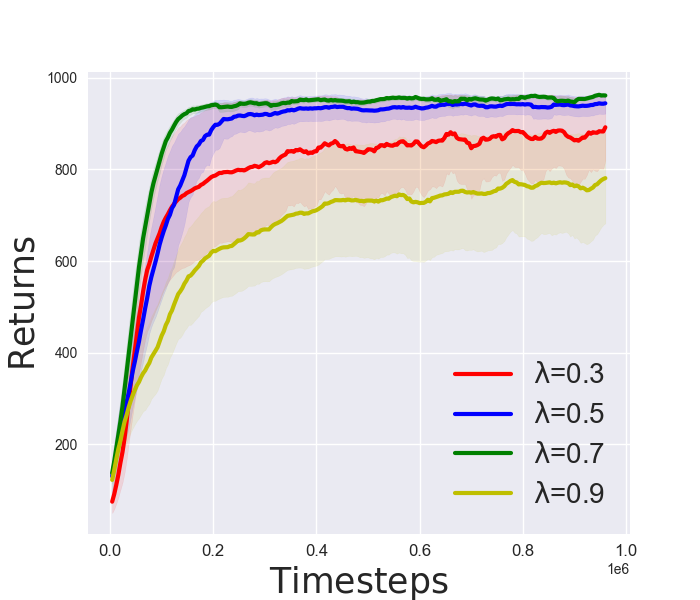}}
    \caption{Ablation study on the sensitivity of Peng's Q($\lambda$) to the hyper-parameter $\lambda$. Each curve corresponds to a choice of $\lambda$ averaged over $5$ random seeds. }
    \label{fig:lambda}
\end{figure}

\end{document}